\documentclass[journal,twoside,web]{ieeecolor}
\usepackage{generic}
\usepackage{cite}
\usepackage{amsmath,amssymb,amsfonts}
\usepackage{algorithmic}

\def\logoname{LOGO_fake}

\usepackage{graphicx} 
\usepackage{textcomp}
\usepackage{color, soul}
\usepackage{mathtools}
\usepackage{bm}
\usepackage[dvipsnames]{xcolor}

\usepackage[hyphens]{url}
\usepackage{hyperref}
\usepackage[hyphenbreaks]{breakurl}

\usepackage{subcaption}
\usepackage{multirow}

\allowdisplaybreaks

\definecolor{dark-gray}{gray}{0.3}

\newcommand{\norm}[1]{\left\lVert#1\right\rVert}
\DeclareMathOperator{\sech}{sech}

\newtheorem{lemma}{Lemma} 
\newtheorem{definition}{Definition} 

\newtheorem{proposition}{Proposition}
\newtheorem{theorem}{Theorem}

\newtheorem{example}{Example}

\begin{document}
\title{Hamiltonian Deep Neural Networks Guaranteeing Non-vanishing Gradients by Design}
\author{Clara Luc\'{i}a Galimberti, Luca Furieri, Liang Xu and Giancarlo Ferrari-Trecate
    \thanks{ Clara Luc\'{i}a Galimberti, Luca Furieri, and Giancarlo Ferrari-Trecate are with the Institute of Mechanical Engineering, École Polytechnique Fédérale de Lausanne, Switzerland. E-mails: {\tt\footnotesize \{clara.galimberti,luca.furieri,  giancarlo.ferraritrecate\}@epfl.ch}. Liang Xu is with the Institute of Artificial Intelligence, Shanghai University, Shanghai, China. E-mail: {\tt \footnotesize liang-xu@shu.edu.cn}.} 
    \thanks{Research supported by the Swiss National Science Foundation under the NCCR Automation (grant agreement 51NF40\textunderscore 180545).}}

\maketitle

\begin{abstract}
    Deep Neural Networks (DNNs) training  can be difficult due to vanishing and exploding gradients during weight optimization through backpropagation.
    To address this problem, we propose a general class of Hamiltonian DNNs (H-DNNs) that stem from the discretization of continuous-time Hamiltonian systems and include several existing DNN architectures based on ordinary differential equations.
    Our main result is that a broad set of H-DNNs ensures non-vanishing gradients by design for an arbitrary network depth. This is obtained by proving that, using a semi-implicit Euler discretization scheme, the backward sensitivity matrices involved in gradient computations are symplectic. 
    We also provide an upper-bound to the magnitude of sensitivity matrices and show that exploding gradients can be controlled through regularization.
    Finally, we enable distributed implementations of backward and forward propagation algorithms in H-DNNs by characterizing appropriate sparsity constraints on the weight matrices. The good performance of H-DNNs is demonstrated on benchmark classification problems, including image classification with the MNIST dataset.
\end{abstract}

\begin{IEEEkeywords}
Deep neural networks, Distributed learning, Hamiltonian systems, ODE discretization
\end{IEEEkeywords}
%Neural networks
%Machine learning
%Nonlinear systems
%Optimization algorithms
% ADD: distributed deep learning

\section{Introduction}
\label{sec:introduction}

Deep learning has achieved remarkable success in various fields like computer vision, speech recognition and natural language processing~\cite{He2016, XiongNLP}.
Within the control community, there is also a growing interest in using DNNs to approximate complex controllers~\cite{Lucia2018ifac, Zoppoli2020book}.
In spite of recent progress, the training of DNNs still presents several challenges such as the occurrence of vanishing or exploding gradients during training based on gradient descent. These phenomena are related to the convergence to zero or the divergence, respectively, of the Backward Sensitivity Matrices (BSMs)\footnote{The BSM of a neural network denotes the sensitivity of the output of the last layer with respect to the output of intermediate layers. Its formal definition can be found in Section~\ref{sec:cost}.} arising in gradient computations through backpropagation. Both situations are very critical as they imply that the learning process either stops prematurely or becomes unstable~\cite{GoodBengCour2016}. 

Heuristic methods for dealing with these problems leverage subtle weight initialization or gradient clipping~\cite{GoodBengCour2016}.
More recent approaches, instead, focus on the study of DNN architectures and associated training algorithms for which exploding/vanishing gradients can be avoided or mitigated \textit{by design}~\cite{RN11052, RN11546, RN11545, RN11544, RN11399 , Haber_2017, haber2017, Chang19, lu2017finite, Weinan2017}.

For instance, in ~\cite{RN11052, RN11545, RN11544} unitary and orthogonal weight matrices are used to control the magnitude of BSMs during backpropagation. Moreover, in~\cite{RN11399, RN11546}, methods based on clipping singular values of weight matrices are utilized to constrain the magnitude of BSMs. These approaches, however, require expensive computations during training~\cite{RN11399, RN11546}, introduce perturbations in gradient descent~\cite{RN11399, RN11546} or use restricted classes of weight matrices~\cite{RN11052, RN11545, RN11544}.

Recently, it has been argued that specific classes of DNNs stemming from the time discretization of Ordinary Differential Equations (ODEs) are less affected by vanishing and exploding gradients~\cite{Haber_2017, haber2017, Chang19, lu2017finite, Weinan2017}.
The arguments provided in~\cite{Haber_2017} rely on the stability properties of the underlying continuous-time nonlinear systems for characterizing relevant behaviors of the corresponding DNNs obtained after discretization.
Specifically, instability of the system results in unstable forward propagation for the DNN model, while convergence to zero of system states is related to the occurrence of vanishing gradients.
This observation suggests using DNN architectures based on dynamical systems that are \textit{marginally stable}, i.e. that produce bounded and non-vanishing state trajectories.
An example is provided by first-order ODEs based on skew-symmetric maps, which have been used in~\cite{Haber_2017, Chang19} for defining \emph{anti-symmetric DNNs}.
Another example is given by dynamical systems in the form
\begin{align}\label{eq:simpleHamiltonianSystems}
\dot{\bf p}=-\nabla_{\bf q} H({\bf p},{\bf q})\,, \quad \dot{\bf q}=\nabla_{\bf p} H({\bf p},{\bf q})\,,
\end{align}
where ${\bf p}, {\bf q}\in\mathbb{R}^n$ and $H(\cdot, \cdot)$ is a Hamiltonian function.
This class of ODEs has motivated the development of \textit{Hamiltonian-inspired} DNNs in~\cite{Haber_2017}, whose effectiveness has been shown in several benchmark classification problems \cite{Haber_2017, Chang19, Chang18a}.

However, these approaches consider only restricted classes of weight matrices or particular Hamiltonian functions, which, together with the specific structure of the dynamics in~\eqref{eq:simpleHamiltonianSystems}, limit the representation power of the resulting DNNs. Moreover, the behavior of BSMs arising in backpropagation has been analyzed only in \cite{Chang19}, which however focuses on DNNs with identical weights in all layers and relies on hard-to-compute quantities such as kinematic eigenvalues \cite{van2004characteristic,BookAscher95}.

The architectures in \cite{Haber_2017, Chang19} and \cite{Chang18a} are conceived for centralized implementations of forward and backward propagation.
As such, they do not cope with the constraints of large networks of geographically distributed nodes with own sensing and computational capabilities, such as Internet of Things devices and multi-vehicle systems for surveillance and scanning tasks. In such scenarios, each node captures a very large stream of input data 
that approaches its memory, bandwidth and battery capabilities. 
However, in order to take system-wide optimal decisions, the measurements gathered by all nodes must be processed simultaneously, which cannot be done in a single location due to physical and computational limitations (see the recent surveys \cite{skala2015scalable,teerapittayanon2017distributed,ben2019demystifying}).  It is therefore important to develop large-scale DNN models for which the training can be distributed between physically separated end devices while guaranteeing satisfactory system-wide predictions. Furthermore, distributed DNN architectures enhance data privacy and fault tolerance~\cite{teerapittayanon2017distributed}, facilitate the learning from graph inputs~\cite{zhou2020graph} and enable the execution of distributed control tasks~\cite{gama2021distributed, yang2021communication}. 

\subsection{Contributions}
The contribution of this paper is fivefold.
First, 
leveraging general models of time-varying Hamiltonian systems~\cite{vanderSchaft2017}, in this work we provide a unified framework for defining H-DNNs, which encompass anti-symmetric \cite{Haber_2017, Chang19} and Hamiltonian-inspired networks~\cite{Haber_2017}.

Second,
for H-DNNs stemming from Semi-Implicit Euler (S-IE) discretization, we prove that the norm of the associated BSMs can never converge to zero irrespective of the network depth and weights. 
This result hinges on the symplectic properties of BSMs,\footnote{The symplectic property used in this paper is rigorously introduced in Definition~\ref{def:symplectic} and is slightly different from standard definitions that can be found in the literature~\cite{deGosson2011book, Hairer2006book}.} and is first shown in the continuous-time setting and then for H-DNNs models. For the result in the discrete-time case, we leverage developments in the field of geometric numerical integration~\cite{Hairer2006book,deGosson2011book}.
%\footnote{The S-IE integration method is also known as Symplectic Euler\cite{Hairer2006book}.}

Third,
we then analyze the phenomenon of exploding gradients. We  construct an example showing that it cannot be avoided for general H-DNNs. However, exploding gradients can be kept under control by including suitable regularization terms in the training cost. 

Forth, we show how to design distributed H-DNN architectures by imposing sparsity constraints on matrix weights. To this purpose, we provide sufficient conditions on the sparsity patterns for distributing forward- and backward-propagation algorithms over computational nodes linked by a communication network.

Finally, 
we provide numerical results on benchmark classification problems demonstrating the flexibility of H-DNNs and showing that thanks to the absence of vanishing gradients, H-DNNs can substantially outperform standard multilayer perceptron (MLP) networks.

At a more conceptual level, our results show the potential of combining methods from system theory (as done in~\cite{RN11084, pauli2021training}) and numerical analysis (as done in \cite{Haber_2017, Chang19} and \cite{Chang18a}) 
for characterizing and analyzing relevant classes of deep networks. 
We also highlight that H-DNNs are fundamentally different from the neural networks proposed in~\cite{RN10758}, which have the same name but are designed to learn the Hamiltonian functions of mechanical systems.

A preliminary version of this work has been presented in the L4DC conference~\cite{ClaraL4DC}.
Compared with~\cite{ClaraL4DC}, this paper is not restricted to H-DNNs stemming from forward Euler discretization. Moreover, for studying exploding/vanishing gradients we do not focus on H-DNNs with constant weights across layers and we do not rely on eigenvalue analysis. Finally, differently from~\cite{ClaraL4DC}, we consider the design of distributed H-DNN architectures.

The remainder of our paper is organized as follows.
H-DNNs are defined in Section~\ref{sec:H-DNN}.
Their properties are analyzed in Section~\ref{sec:CTAnalysis} and Section~\ref{sec:DT_analysis} from a continuous-time and a discrete-time perspective, respectively.
Numerical examples are discussed in Section~\ref{sec:numerical_experiments} and
concluding remarks are provided in Section~\ref{sec:Conclusion}.

\subsection{Notation}
We use $0_{m \times n}$ ($1_{m \times n}$) to denote the matrix of all zeros (all ones) of dimension $m \times n$,
$I_n$ to denote the identity matrix of size $n \times n$,  
$0_n$ to denote the square zero matrix of dimension $n \times n$ and  $1_n$ to denote the vector of all ones of length $n$.
For a vector $\mathbf{x} \in \mathbb{R}^n$,  $\mathrm{diag}({\bf x})$ is the $n \times n$ diagonal matrix with the elements of ${\bf x}$ on the diagonal. For vectors $\mathbf{x} \in \mathbb{R}^n$, $\mathbf{y} \in \mathbb{R}^m$, we denote the vector $\mathbf{z} \in \mathbb{R}^{m+n}$ stacking them one after the other as $\mathbf{z} = (\mathbf{x},\mathbf{y})$. 
We adopt the denominator layout for derivatives, that is, the derivative of $\mathbf{y}\in\mathbb{R}^m$ with respect to $\mathbf{x}\in\mathbb{R}^n$ is 
$\frac{\partial \mathbf{y}}{\partial \mathbf{x}} \in \mathbb{R}^{n \times m}$.

\section{Hamiltonian Deep Neural Networks}\label{sec:H-DNN}

This section, besides providing a short introduction to DNNs defined through the discretization of nonlinear systems, presents all the ingredients needed for the definition and implementation of H-DNNs. Throughout the paper, we focus on classification tasks since they have been used as benchmarks for similar architectures (see Section~\ref{sec:numerical_experiments}). However, the main results apply to regression tasks as well, which only require modifying the output layer of the network.  We also illustrate how H-DNNs generalize several architectures recently appeared in the literature. 
Finally, we introduce the problem of vanishing/exploding gradients, which is analyzed in the rest of the paper.

\subsection{DNN induced by ODE discretization} \label{sec:ODEinspirednet}

We consider the first-order nonlinear dynamical system
\begin{equation}
\dot{\bf y}(t) = {\bf f}({\bf y}(t), {\boldsymbol{\theta}}(t))\,, \quad 0 \leq t \leq T\,,
\label{eq:firstorderODE}
\end{equation}
where $ {\bf y}(t)\in \mathbb{R}^n$, $ {\bf y}(0) = {\bf y}_0$ and  ${\boldsymbol{\theta}}(t)\in \mathbb{R}^{n_\theta}$ is a vector of parameters. For specifying a DNN architecture, we discretize \eqref{eq:firstorderODE} with sampling period $h = \frac{T}{N}$, $N \in \mathbb{N}$ and utilize the resulting  discrete-time equations for defining each of the $N$ network layers~\cite{chen2018neural}. For instance, using Forward Euler (FE) discretization, one obtains
\begin{equation}
{\bf y}_{j+1} = {\bf y}_j + h\, {\bf f}({\bf y}_j, {\boldsymbol{\theta}}_j)\,, \quad j=0,1,\dots,N-1\,.
\label{eq:firstorderODE_td}
\end{equation}
The above equation can be seen as the model of a residual neural network~\cite{He2016}, where ${\bf y}_j$ and ${\bf y}_{j+1} \in \mathbb{R}^{n}$ represent the input and output of layer $j$, respectively.

Clearly, it may be convenient to replace FE with more sophisticated discretization methods, depending on the desired structural properties of the DNN. For instance, the authors of ~\cite{Haber_2017} use Verlet discretization for a specific class of DNNs. Later in this work, we show that S-IE discretization will allow us to formally prove that the phenomenon of vanishing gradients cannot occur in H-DNNs. 

A remarkable feature of ODE-based DNNs is that their properties can be studied by using nonlinear system theory for analyzing the continuous-time model \eqref{eq:firstorderODE}. Further, this allows one to study the effect of discretization independently. 

\subsection{From Hamiltonian dynamics to H-DNNs}

We consider the neural network architectures inspired by  time-varying Hamiltonian systems~\cite{Guo2006,vanderSchaft2017} defined as
\begin{equation}
\dot{\mathbf{y}}(t) = {\bf J}(t)  \frac{\partial H({\bf y}(t),t)}{\partial {\bf y}(t)}\,, \quad {\bf y}(0) = {\bf y}_0\,,
\label{eq:TV_HS}
\end{equation}
where ${\bf J}(t) \in \mathbb{R}^{n\times n}$ is skew-symmetric i.e. ${\bf J}(t) = -{\bf J}^\top(t)$ at all times and the continuously differentiable function $H:\mathbb{R}^n \times \mathbb{R}\rightarrow \mathbb{R}$  is the \emph{Hamiltonian function}.

In order to recover the DNNs proposed in \cite{Haber_2017, Chang19, Chang18a}, we consider the following Hamiltonian function
\begin{equation}
H({\bf y}(t),t) = \left[ \tilde{\sigma}({\bf K}(t) {\bf y}(t) + {\bf b}(t)) \right]^\top 1_{n}\,,
\label{eq:nlH}
\end{equation}
where $\tilde{\sigma}:\mathbb{R}\rightarrow \mathbb{R}$ is a differentiable map, applied element-wise when the argument is a matrix, and the derivative of $\tilde{\sigma}(\cdot)$ is called  \emph{activation function} $\sigma(\cdot)$. Specifically, as it is common for neural networks, we consider  activation functions $\sigma:\mathbb{R}\rightarrow \mathbb{R}$ that are differentiable almost everywhere and such that
\begin{equation}
    \label{eq:sigma_bound}
    |\sigma'(x)|\leq S\,,
\end{equation}
for some $S>0$, where $\sigma'(x)$ denotes any sub-derivative. This assumption holds for  common activation functions such as  $\tanh(\cdot)$, $\text{ReLU}(\cdot)$, and the logistic function.
Notice that
\begin{align}
\frac{\partial H({\bf y}(t),t)}{\partial {\bf y}(t)}  &= \frac{\partial ({\bf K}(t) {\bf y}(t) + {\bf b}(t))}{\partial {\bf y}(t)}  \frac{\partial H({\bf y}(t),t)}{\partial ({\bf K}(t) {\bf y}(t) + {\bf b}(t))} \nonumber\\
&= {\bf K}^\top(t) \sigma({\bf K}(t) {\bf y}(t) + {\bf b}(t)) \,,\label{eq:partialH}
\end{align}
and, therefore, system~\eqref{eq:TV_HS} can be rewritten as
\begin{equation}
\dot{\bf y}(t) = {\bf J}(t) {\bf K}^\top(t) \sigma( {\bf K}(t) {\bf y}(t) + {\bf b}(t) ) \,, \quad {\bf y}(0) = {\bf y}_0\,.
\label{eq:ODE_H}
\end{equation}

The ODE \eqref{eq:ODE_H} serves as the basis for defining H-DNNs. Indeed, as outlined in the previous subsection, a given discretization scheme for \eqref{eq:ODE_H} naturally leads to a neural network architecture. In the context of this work, we focus on FE and S-IE discretizations, resulting in the following DNN architectures.

\vspace{0.1cm}

\noindent\textbf{H$_1$-DNN:} By discretizing \eqref{eq:ODE_H} with FE we obtain the layer equation
\begin{equation}
{\bf y}_{j+1} = {\bf y}_j + h\, {\bf J}_j\, {\bf K}^\top_j \sigma({\bf K}_j {\bf y}_j +{\bf b}_j)\,,
\label{eq:H-DNN_fE}
\end{equation}
which can be interpreted as the Hamiltonian counterpart of~\eqref{eq:firstorderODE_td}. 

\vspace{0.1cm}

\noindent\textbf{H$_2$-DNN:} Assume that the number of features $n\in \mathbb{N}$ is even\footnote{This condition can be always fulfilled by performing feature augmentation~\cite{dupont2019augmented}.} and split the feature vector as $\mathbf{y}_j = ({\bf p}_j, {\bf q}_j) $ for  $j= 0,\dots,N$ where ${\bf p}_j, {\bf q}_j \in \mathbb{R}^{\frac{n}{2}}$. Further, assume that $\mathbf{J}_j = \mathbf{J}$ does not vary across layers. Then, S-IE discretization of \eqref{eq:ODE_H} leads to the layer equation 
\begin{align}\nonumber
\begin{bmatrix}
{\bf p}_{j+1} \\ {\bf q}_{j+1}
\end{bmatrix}
&=
\begin{bmatrix}
{\bf p}_{j} \\ {\bf q}_{j}
\end{bmatrix}
+
h\,
{\bf J}
\begin{bmatrix}
\frac{\partial H}{\partial{\bf p}}({\bf p}_{j+1},{\bf q}_{j}, t_j) \\ \frac{\partial H}{\partial{\bf q}}({\bf p}_{j+1},{\bf q}_{j}, t_j)
\end{bmatrix}\\
&=
\begin{bmatrix}
{\bf p}_{j} \\ {\bf q}_{j}
\end{bmatrix}
+
h\,
{\bf J}
{\bf K}^\top_j \sigma\left({\bf K}_j 
\begin{bmatrix}
{\bf p}_{j+1} \\ {\bf q}_{j}
\end{bmatrix}
+{\bf b}_j\right)\,,
\label{eq:implicit_euler}
\end{align}
where $j=0,1,\dots, N-1$.
In general, computing the updates $(\mathbf{p}_{j+1},\mathbf{q}_{j+1})$ as per \eqref{eq:implicit_euler} involves solving a system of implicit equations.
For computational aspects of deep learning with implicit layers, we refer the reader to~\cite{RN11553, bai2019deep}.
Since implicit equations can be hard to solve, to make the updates \eqref{eq:implicit_euler} easily computable one can further assume that
\begin{equation}\label{eq:implicit_euler_JKb}
{\bf J} \hspace{-0,1cm} = \hspace{-0,1cm} \begin{bmatrix}
0_\frac{n}{2} & -{\bf X}^\top\\
{\bf X} & 0_\frac{n}{2} 
\end{bmatrix},\,
{\bf K}_j = \begin{bmatrix}
{\bf K}_{p,j} & 0_\frac{n}{2} \\
0_\frac{n}{2} & {\bf K}_{q,j}
\end{bmatrix},\,
{\bf b}_j = \begin{bmatrix}
{\bf b}_{p,j} \\
{\bf b}_{q,j}
\end{bmatrix},
\end{equation}
which yields the layer equations 
\begin{align}
	{\bf p}_{j+1} &= {\bf p}_j - h {\bf X}^\top {{\bf K}_{q,j}}^\top \sigma({\bf K}_{q,j}{\bf q}_{j} + {\bf b}_{q,j})\,, \label{eq:p_update}\\
	{\bf q}_{j+1}	&= {\bf q}_j + h {\bf X} {{\bf K}_{p,j}}^\top \sigma({\bf K}_{p,j} {\bf p}_{j+1} + {\bf b}_{p,j} )\,. \label{eq:q_update}
\end{align}
It is easy to see that one can first compute $\mathbf{p}_{j+1}$ through \eqref{eq:p_update}, while  $\mathbf{q}_{j+1}$ is obtained as a function of $\mathbf{p}_{j+1}$ through \eqref{eq:q_update}.\footnote{The layer equations \eqref{eq:p_update}-\eqref{eq:q_update} are analogous to those obtained in \cite{Haber_2017} and \cite{Chang18a} by using Verlet discretization.}

\vspace{0.1cm}

H$_1$-DNNs are motivated by the simplicity of the FE discretization scheme and will be compared in Section~\ref{sec:numerical_experiments} with existing DNNs proposed in \cite{Haber_2017} and \cite{Chang18a} using benchmark examples. However, even if system \eqref{eq:ODE_H} is marginally stable,\footnote{This is always the case for constant parameters ${\bf J}$, ${\bf K}$ and ${\bf b}$, see \cite{vanderSchaft2017}.} FE discretization might introduce instability and lead to layer features $\mathbf{y}_j$ that grow exponentially with the network depth~\cite{ClaraL4DC}.
Instead, H$_2$-DNNs do not suffer from this problem because, as shown in \cite{Hairer2006book}, S-IE discretization preserves the marginal stability of \eqref{eq:ODE_H}.  More importantly, as we show  in Section~\ref{sec:DT_analysis},  H$_2$-DNN architectures completely prevent the phenomenon of vanishing gradients.

We  conclude this subsection by highlighting that H$_1$- and H$_2$-DNNs are more general than the DNN architectures proposed in \cite{Haber_2017, Chang19, Chang18a}. A precise comparison is provided in Appendix~\ref{ap:relation_existing_arq}. 

\subsection{Training of H-DNNs}\label{sec:cost}
Similarly to \cite{Haber_2017, Chang19, Chang18a}, we consider multicategory classification tasks based on the training set  $\{({\bf y}_0^k,c^k), k=1,\dots,s\}$, where $s$ denotes the number of examples, ${\bf y}_0^k$ are the feature vectors, and $c^k\in\{1,\ldots, n_c\}$ are the corresponding labels.
As standard in classification through DNNs, the architectures~\eqref{eq:H-DNN_fE} and~\eqref{eq:implicit_euler} are complemented with an output layer ${\bf y}_{N+1} = {\bf f}_N({\bf y}_{N}, {\boldsymbol{\theta}}_{N})$ composed, e.g., by the softmax function, to re-scale elements of ${\bf y}_{N}$ for representing class membership probabilities~\cite{GoodBengCour2016}.
H-DNNs are trained by solving the following empirical risk minimization problem
\begin{align}
\min_{\boldsymbol{\theta}} & \qquad  \frac{1}{s} \sum_{k=1}^s \mathcal{L}({\bf f}_N({\bf y}^k_{N},\bm{\theta}_N), c^k) + R(\boldsymbol{\theta}) \label{eq:minimization}\\
\text{s.t.} & \text{ \eqref{eq:H-DNN_fE} or \eqref{eq:implicit_euler}}, \quad j=0,1,\dots,N-1\,,\nonumber
\end{align}
where $\boldsymbol{\theta}$ denotes trainable parameters, i.e., 
$\boldsymbol{\theta} = \boldsymbol{\theta}_{0,\dots,N}$ with 
$\boldsymbol{\theta}_{j} = \{{\bf J}_{j}, {\bf K}_{j}, {\bf b}_{j}\}$ for $j=0,\dots,N-1$, and $R(\boldsymbol{\theta})$ is a regularization term given by
$R(\boldsymbol{\theta}) = \alpha \, R_K({\bf K}_{0,\dots,N-1}, {\bf b}_{0,\dots,N-1}) + \alpha_{\ell} R_{\ell}(\boldsymbol{\theta}_{0,\dots,N-1}) + \alpha_N R_N(\boldsymbol{\theta}_N)$. 
The term $R_K$ is defined as  $\frac{h}{2}\sum_{j=1}^{N-1} \left( \left\| {\bf K}_j-{\bf K}_{j-1} \right\|^2_F + \left\| {\bf b}_j-{\bf b}_{j-1} \right\|^2\right)$, which follows the work in \cite{Haber_2017} and \cite{Chang18a}, and favours smooth weight variations across consecutive layers.
The terms $R_{\ell}(\cdot)$ and $R_N(\cdot)$ refer to a standard $L_2$ regularization for the inner layers and the output layer.
The coefficients $\alpha \geq 0$, $\alpha_{\ell} \geq 0$ and $\alpha_N \geq 0$\footnote{$\alpha_\ell$ and $\alpha_N$ are usually called \textit{weight decays}.} are hyperparameters representing the trade-off between fitting and regularization \cite{Haber_2017}.

To minimize the cost \eqref{eq:minimization}, it is common to utilize gradient descent, which can steer the parameters to a stationary point $\bm{\theta}^\star$  of the cost such that
\begin{equation}
\label{eq:stationary_points}
\nabla_{\bm{\theta}}(\mathcal{L}(\bm{\theta}^\star)+R(\bm{\theta}^\star)) = 0\,.
\end{equation}

When using gradient descent to minimize
\eqref{eq:minimization},  at each iteration the gradient of the loss function $\mathcal{L}$ with respect to the parameters needs to be calculated. 
This gradient, for parameter $i$ of layer $j$ is computed according to the chain rule as:
\begin{align}
\frac{\partial \mathcal{L}}{\partial \theta_{i,j}} = \frac{\partial {\bf y}_{j+1}}{\partial \theta_{i,j}} \frac{\partial \mathcal{L}}{\partial {\bf y}_{j+1}}  =  \frac{\partial {\bf y}_{j+1}}{\partial \theta_{i,j}} 
\left( \prod_{l=j+1}^{N-1} \frac{\partial {\bf y}_{l+1}}{\partial {\bf y}_l} \right)
\frac{\partial \mathcal{L}}{\partial {\bf y}_{N}}\,.\label{eq:parameterGradient}
\end{align}

It is clear from \eqref{eq:parameterGradient} that the gradient of the loss with respect to the parameters depends directly on the quantities
\begin{equation}
\boldsymbol{\delta}_j = \frac{\partial \mathcal{L}}{\partial {\bf y}_j}\quad j=1,2,\dots,N \,,
\end{equation}
which can be computed iteratively as follows.

\noindent\textbf{H$_1$-DNN \eqref{eq:H-DNN_fE}:} %For the H$_1$-DNN~ in \eqref{eq:H-DNN_fE},
\begin{equation}
\label{eq:FE_BP}
\boldsymbol{\delta}_{j} =  (I_n + h{\bf K}_{j}^\top \text{diag}(\sigma'({\bf K}_{j} {\bf y}_{j} + {\bf b}_{j})) {\bf K}_{j} {\bf J}_{j}^\top)  \,\boldsymbol{\delta}_{j+1},
\end{equation}

\noindent\textbf{H$_2$-DNN \eqref{eq:p_update}-\eqref{eq:q_update}:} 
let us
define  $\boldsymbol{\gamma}_j = \frac{\partial \mathcal{L}}{\partial {\bf p}_j}$, $\boldsymbol{\lambda}_j = \frac{\partial \mathcal{L}}{\partial {\bf q}_j}$, i.e.,  $\boldsymbol{\delta}_j=(\boldsymbol{\gamma}_j, \boldsymbol{\lambda}_j) $.
One has
\begin{align}
    \boldsymbol{\gamma}_{j} =& \boldsymbol{\gamma}_{j+1} + h {{\bf K}}_{p,j}^\top \text{diag}(\sigma'({\bf K}_{p,j} {\bf p}_{j+1} + {\bf b}_{p,j})) {\bf K}_{p,j} {\bf X}^\top \,\, \boldsymbol{\lambda}_{j+1} \nonumber\\
    \boldsymbol{\lambda}_{j} =& \boldsymbol{\lambda}_{j+1} - h {{\bf K}^\top_{q,j}} \text{diag}(\sigma'({\bf K}_{q,j}{\bf q}_j + {\bf b}_{q,j})) {\bf K}_{q,j}{\bf X} \,\, \boldsymbol{\gamma}_{j} \label{eq:backward_implicit_euler}
\end{align}
for $j=N-1,\dots,1$.

Throughout the paper, we will refer to the matrix 
\begin{equation}
\label{eq:BSM}
\frac{\partial \mathbf{y}_N}{\partial \mathbf{y}_{N-j}} = \prod_{l=N-j}^{N-1} \frac{\partial {\bf y}_{l+1}}{\partial {\bf y}_l}\,,
\end{equation} 
as the BSM at layer $N-j$, 
for $j=1,\dots,N-1$,\footnote{For $j=0$, the BSM is $\frac{\partial \mathbf{y}_N}{\partial \mathbf{y}_{N}} = I_n$.}
which, according to \eqref{eq:parameterGradient}, 
allows one to compute the partial derivatives $\frac{\partial \mathcal{L}}{\partial \theta_{i,N-j-1}}$ 
at every layer $N-j-1$. 
As shown in the next subsection, this quantity is the key to studying the phenomena of vanishing and exploding gradients.

\subsection{Vanishing/exploding gradients}\label{sec:VanExpGradient}

Gradient descent methods for solving \eqref{eq:minimization} update the vector ${\boldsymbol{\theta}}$ as
\begin{equation}\label{eq:GD_update}
\boldsymbol{\theta}^{(k+1)} = \boldsymbol{\theta}^{(k)} - \gamma \cdot \nabla_{\boldsymbol{\theta}^{(k)}}
\mathcal{L}\,,
\end{equation}
where $k$ is the iteration number, $\gamma >0$ is the optimization step size
and the elements of $\nabla_{\boldsymbol{\theta}}\mathcal{L}$ are given in~\eqref{eq:parameterGradient}.

The problem of vanishing/exploding gradients is  related to the BSM \eqref{eq:BSM}. Indeed, when $\|\frac{\partial \mathbf{y}_N}{\partial \mathbf{y}_{N-j}}\|$ is very small, from \eqref{eq:parameterGradient} the gradients $\frac{\partial \mathcal{L}}{\partial \theta_{i,N-j-1}}$ vanish despite not having reached a stationary point, and the training stops  prematurely.

Vice-versa, if $\|\frac{\partial \mathbf{y}_N}{\partial \mathbf{y}_{N-j}}\|$ is very large, the derivative $\frac{\partial \mathcal{L}}{\partial \theta_{i,N-j-1}}$ becomes very sensitive to perturbations in the vectors $\frac{\partial {\bf y}_{N-j}}{\partial \theta_{i,N-j-1}}$ and $\frac{\partial \mathcal{L}}{\partial {\bf y}_N}$, and this can make the learning process unstable or cause overflow issues.
Both problems are generally exacerbated when the number of layers $N$ is large~\cite{GoodBengCour2016}.

In Sections \ref{sec:CTAnalysis} and \ref{sec:DT_analysis} we analyze in detail the properties of BSMs,
with the goal of showing that vanishing gradients cannot occur while exploding gradients can be mitigated. 
To this purpose, it is convenient to first adopt the continuous-time perspective enabled by system \eqref{eq:TV_HS}. 

\section{Continuous-Time Analysis}\label{sec:CTAnalysis}

In this section, we analyze the properties of H-DNNs from a continuous-time point of view. 
First, by using backward sensitivity analysis, we derive a continuous-time representation of the backpropagation algorithm. Specifically, the continuous-time counterpart of the BSM \eqref{eq:BSM} is characterized as the solution to a Linear Time-Varying (LTV) ODE. 
Second, we prove that continuous-time BSMs are lower-bounded in norm by the value $1$, independent of the network depth and of the choice for the time-varying weights $\mathbf{K}(t)$ and $\mathbf{b}(t)$. 
Third, we observe that, contrary to what has been conjectured in previous work \cite{Chang19}, the gradients of general Hamiltonian networks  may explode with the network depth, even if the weights are the same in all layers. 
This phenomenon is shown by providing an explicit example.
To mitigate this issue, we derive an informative upper-bound on the norm of the continuous-time BSMs that holds for the general H-DNN architecture. 
This bound suggests utilizing a regularizer on the norms for the weights during training. 
Last, motivated by large-scale applications, we conclude the section by showcasing how all relevant properties of H-DNNs are naturally ported to a distributed learning setup, by appropriately constraining the sparsities for the network weights.

\subsection{Continuous-time backward sensitivity analysis}
As discussed in Section~\ref{sec:VanExpGradient}, the phenomena of vanishing and exploding gradients are tightly linked to the behavior of the BSM \eqref{eq:BSM}.
The continuous-time counterpart of \eqref{eq:BSM} that we study in this section is given by
\begin{equation}
    \label{eq:norm_grad_CT}
\frac{\partial {\bf y}(T)}{\partial {\bf y}(T-t)} \,.
\end{equation}
The following Lemma, whose proof can be found in Appendix~\ref{sec:lem:grad_dynamics}, expresses the backward sensitivity dynamics of H-DNNs as the solution to an LTV ODE.

\begin{lemma}\label{lem:grad_dynamics}
    Given the ODE \eqref{eq:ODE_H} associated with an H-DNN, the {continuous-time} backward sensitivity matrix $\frac{\partial {\bf y}(T)}{\partial {\bf y}(T-t)}$ verifies 
	\begin{equation}
	\frac{d}{d t} \frac{\partial {\bf y}(T)}{\partial {\bf y}(T-t)} = {\bf A}(T-t) \frac{\partial {\bf y}(T)}{\partial {\bf y}(T-t)}\,,
	\label{eq:BackwardGradientDynamics}
	\end{equation}
	where $t \in [0,T]$ and ${\bf A}(\tau) = {\bf K}^\top(\tau) {\bf D}({\bf y}(\tau), \tau) {\bf K}(\tau){\bf J}^\top(\tau)$,
	with ${\bf D}({\bf y}(\tau),\tau) = \text{diag}\left(\sigma'({\bf K}(\tau) {\bf y}(\tau) +{\bf b}(\tau))\right)$.
\end{lemma}
      
\vspace{0.2cm}
      
Accordingly, the continuous-time counterpart of \eqref{eq:parameterGradient}  is given by
\begin{equation*}
    \frac{\partial \mathcal{L}}{ \partial \theta_{i}(T-t)} = \frac{\partial \mathbf{y}(T-t)}{\partial \theta_{i}(T-t)}\bm{\delta}(T-t)\,,
\end{equation*}
where $\bm{\delta}(T-t) = \frac{\partial \mathcal{L}}{\partial \mathbf{y}(T-t)}$ is the solution to the backward-in-time ODE
\begin{multline}
    \label{eq:delta_CT}
    \dot{\bm{\delta}}(T-t) = {\bf K}^\top(T-t) {\bf D}({\bf y}(T-t), T-t) {\bf K}(T-t) \cdot \\ \cdot {\bf J}^\top(T-t) \bm{\delta}(T-t)\,,
\end{multline}
initialized with $\bm{\delta}(T) = \frac{\partial \mathcal{L}}{\partial\mathbf{y}(T)}$.

\vspace{0.1cm}

By Lemma~\ref{lem:grad_dynamics} the phenomena of vanishing and exploding gradients are avoided if the LTV system \eqref{eq:BackwardGradientDynamics} is {marginally} stable, i.e. its solutions are neither diverging nor asymptotically converging to zero.
 
When $\bm{\theta}(t) = \bm{\theta}$ for every $t \in [0,T]$, the matrix ${\bf A}(T-t)$ has all eigenvalues on the imaginary axis~\cite{Chang19}. 
Our work \cite{ClaraL4DC} has further shown that  ${\bf A}(T-t)$ is diagonalizable. While \cite{Chang19} suggests that these spectral properties of $\mathbf{A}(T-t)$ may lead to non-vanishing and non-exploding gradients when  $\mathbf{y}(t)$ varies slowly enough, to the best of the authors' knowledge, there is no direct link between how fast ${\bf A}(T-t)$ varies and the stability of $\frac{\partial {\bf y}(T)}{\partial {\bf y}(T-t)}$.
We will show this fact through an example later in this section.

Indeed, as opposed to linear time-invariant (LTI) systems, the stability of \eqref{eq:BackwardGradientDynamics} cannot be determined solely based on the eigenvalues of the time-varying  matrix  ${\bf A}(T-t)$~\cite{wu1974note}.  
As pointed out in \cite{Chang19}, a rigorous analysis of the properties of ${\bf A}(T-t)$ can be conducted by using the notion of \emph{kinematic eigenvalues}~\cite{van2004characteristic,BookAscher95}. 
These are determined by finding a time-varying transformation that diagonalizes ${\bf A}(T-t)$. However, such transformation depends explicitly on the LTV ODE solution $\frac{\partial {\bf y}(T)}{\partial {\bf y}(T-t)}$, which makes the computation of kinematic eigenvalues as hard as solving \eqref{eq:BackwardGradientDynamics}.

Motivated as above, rather than studying the properties of ${\bf A}(T-t)$, in this paper we analyze the properties of $\frac{\partial {\bf y}(T)}{\partial {\bf y}(T-t)}$ directly.

\subsection{Non-vanishing BSM}\label{sec:symplecticity_CT}
Our first main result is to establish that, under the assumption that $\mathbf{J}(t) = \mathbf{J}$ is constant for all $t$,  $\frac{\partial {\bf y}(T)}{\partial {\bf y}(T-t)}$ is  a \emph{symplectic matrix} with respect to $\mathbf{J}$. 
\begin{definition}[Symplectic matrix]
\label{def:symplectic}
Let $\mathbf{Q} \in \mathbb{R}^{n \times n}$ be a skew-symmetric matrix, i.e. $\mathbf{Q}+\mathbf{Q}^\top = 0_{n}$. A matrix $\mathbf{M}$ is symplectic with respect to $\mathbf{Q}$ if
\begin{equation*}
    \mathbf{M}^\top \mathbf{Q} \mathbf{M} = \mathbf{Q}\,.
\end{equation*}
\end{definition}
Symplectic matrices are usually defined by assuming 
that 
$\mathbf{Q} = \begin{bsmallmatrix}0 & I \\-I &0 \end{bsmallmatrix} \in \mathbb{R}^{n \times n}$ 
with $n$ being an even integer~\cite{Hairer2006book}.
In this respect, Definition~\ref{def:symplectic} provides a slightly generalized notion of symplecticity.
\begin{lemma}
\label{le:generalized_symplecticity}
    Consider an H-DNN as per \eqref{eq:ODE_H} with $\mathbf{J}( t) = \mathbf{J}$ for all $t \in [0,T]$, where $\mathbf{J}$ is any skew-symmetric matrix. Then
    $\frac{\partial {\bf y}(T)}{\partial {\bf y}(T-t)}$ is symplectic with respect to $\mathbf{J}$, i.e.
    \begin{equation}
    \label{eq:generalized_symplecticity}
        \left(\frac{\partial {\bf y}(T)}{\partial {\bf y}(T-t)}\right)^\top \mathbf{J}\frac{\partial {\bf y}(T)}{\partial {\bf y}(T-t)} = \mathbf{J}\,,
    \end{equation}
        for all $t \in [0,T]$.
\end{lemma}
\begin{proof}
 For brevity, let $\Phi = \frac{\partial {\bf y}(T)}{\partial {\bf y}(T-t)}$, $\tau = T-t$ and  $\mathbf{D}(\tau) = \text{diag}\left(\sigma'(\mathbf{K}(\tau)\mathbf{y}(\tau)+\mathbf{b}(\tau))\right)$.
 We have
 \begin{align*}
     &\frac{d}{dt}\left(\Phi^\top\mathbf{J}\Phi\right) \\
     &= \dot{\Phi}^\top\mathbf{J}\Phi +\Phi^\top\mathbf{J}\dot{\Phi}\\
     &=\Phi^\top\mathbf{J}\mathbf{K}^\mathsf{T}(T\hspace{-0.1cm}-t)\mathbf{D}(T\hspace{-0.1cm}-t)\mathbf{K}(T\hspace{-0.1cm}-t)\mathbf{J}\Phi\\
     &~~+\Phi^\top\mathbf{J}\mathbf{K}^\mathsf{T}(T\hspace{-0.1cm}-t)\mathbf{D}(T\hspace{-0.1cm}-t)\mathbf{K}(T\hspace{-0.1cm}-t)\mathbf{J}^\top\Phi\\
     &= \Phi^\top\mathbf{J}\mathbf{K}^\mathsf{T}(T\hspace{-0.1cm}-t)\mathbf{D}(T\hspace{-0.1cm}-t)\mathbf{K}(T\hspace{-0.1cm}-t)\mathbf{J}\Phi\\
     &~~-\Phi^\top\mathbf{J}\mathbf{K}^\mathsf{T}(T\hspace{-0.1cm}-t)\mathbf{D}(T\hspace{-0.1cm}-t)\mathbf{K}(T\hspace{-0.1cm}-t)\mathbf{J}\Phi
     = 0_{n}\,.
 \end{align*}
  Since $\frac{\partial {\bf y}(T)}{\partial {\bf y}(T)} = I_n$ by definition, then $\left(\frac{\partial {\bf y}(T)}{\partial {\bf y}(T)}\right)^\top\mathbf{J}\frac{\partial {\bf y}(T)}{\partial {\bf y}(T)} = \mathbf{J}$. 
  As the time-derivative of $\frac{\partial {\bf y}(T)}{\partial {\bf y}(T-t)}^\top \mathbf{J} \frac{\partial {\bf y}(T)}{\partial {\bf y}(T-t)}$ is equal to zero for every $t \in [0,T]$, then $\frac{\partial {\bf y}(T)}{\partial {\bf y}(T-t)}^\top \mathbf{J} \frac{\partial {\bf y}(T)}{\partial {\bf y}(T-t)} = \mathbf{J}$ for all $t \in [0,T]$. 
\end{proof}

We highlight that Lemma~\ref{le:generalized_symplecticity} is an adaptation of Poincar{\'e} theorem~\cite{poincare1899methodes,Hairer2006book} to the case of the time-varying Hamiltonian functions \eqref{eq:nlH} and the notion of symplecticity provided in Definition~\ref{def:symplectic}. 

Next, we exploit Lemma~\ref{le:generalized_symplecticity}  
to prove that the norm of $\frac{\partial {\bf y}(T)}{\partial {\bf y}(T-t)}$ cannot vanish for all $t\in[0,T]$.
\begin{theorem}
\label{th:nonvanishing_gradients}
    Consider an H-DNN as per \eqref{eq:ODE_H} with $\mathbf{J}( t) = \mathbf{J}$ for all $t \in [0,T]$, where $\mathbf{J}$ is any non-zero skew-symmetric matrix. Then
    \begin{equation}
    \label{eq:nonvanishing_gradients}
        \norm{\frac{\partial {\bf y}(T)}{\partial {\bf y}(T-t)}} \geq 1\,,
    \end{equation}
     for all $t \in [0,T]$, where $\norm{\cdot}$ denotes any sub-multiplicative norm.
\end{theorem}
\begin{proof}
We know by Lemma~\ref{le:generalized_symplecticity} that \eqref{eq:generalized_symplecticity} holds. Hence, we have
\begin{equation*}
     \|\mathbf{J}\| = \norm{\left(\frac{\partial {\bf y}(T)}{\partial {\bf y}(T-t)}\right)^\top \mathbf{J} \frac{\partial {\bf y}(T)}{\partial {\bf y}(T-t)}} \leq \norm{\frac{\partial {\bf y}(T)}{\partial {\bf y}(T-t)}}^2 \norm{\mathbf{J}}\,,
 \end{equation*}
 for all $t \in [0,T]$. The above inequality implies the result.
\end{proof}

 %The result of Theorem~\ref{th:nonvanishing_gradients} is crucial to prove, in Section~\ref{sec:DT_analysis}, that the BSM \eqref{eq:BSM} is symplectic for H$_2$-DNN network implementations. Hence, all H$_2$-DNNs enjoy non-vanishing gradients.

\subsection{Towards non-exploding gradients}\label{sec:non_exploding}

The next question is whether $\norm{\frac{\partial {\bf y}(T)}{\partial {\bf y}(T-t)}}$ may diverge to infinity in general as $T$ increases. It was conjectured in \cite{Chang19} that, when all the weights $\bm{\theta}(t) = \bm{\theta}$ are time-invariant and $\mathbf{y}(t)$ varies slowly enough, the gradients will not explode for arbitrarily deep networks. 
Next, we show through a simple example that, unfortunately, gradients may explode even under such conditions. Nevertheless, we derive an informative upper-bound for mitigating this issue.

\subsubsection{An example of exploding gradients}
We consider the two-dimensional H-DNN whose forward equation is given by
\begin{equation}
\label{eq:2D_Example}
    \dot{\mathbf{y}}(t) = \epsilon\mathbf{J} \tanh\left(\mathbf{y}(t)\right)\,,
\end{equation}
where $\mathbf{J} = \begin{bmatrix}0 &-1\\1&0\end{bmatrix}$ and $\epsilon\in \mathbb{R}$. Clearly, \eqref{eq:2D_Example} is an instance of \eqref{eq:ODE_H} where $\sigma(\cdot) = \tanh(\cdot)$, $\mathbf{K}(t) = I_2$, $\mathbf{b}(t) = 0_{2 \times 1}$ and $\mathbf{J}(t) = \mathbf{J}$. Furthermore, \eqref{eq:2D_Example} is an instance of antisymmetric network from \cite{Chang19} without the input term.
As in \cite{Chang19}, all the weights are time-invariant, and if $|\epsilon|$ is very small, then $\mathbf{y}(t)$ varies arbitrarily slowly. It can be shown that, when the weights are time-invariant, the level of the Hamiltonian remains constant for all $t\in [0,T]$, that is, $H(\mathbf{y}(t)) = H(\mathbf{y}(0))$ \cite{Haber_2017,ClaraL4DC}. Further, we prove here that the solution $\mathbf{y}(t)$ to the ODE \eqref{eq:2D_Example} is periodic and that its period increases or decreases as the Hamiltonian energy $H(\mathbf{y}(0))$ increases or decreases, respectively. The proof is reported in Appendix~\ref{sec:le:periodic}. 

In this subsection, we denote the solution to the ODE \eqref{eq:2D_Example} at time $t$, initialized at time $t_0$ with initial condition $\mathbf{y}(t_0) = \mathbf{y}_0$ as $\mathbf{s}(t,t_0,\mathbf{y}_0)$.
\begin{lemma}
\label{le:periodic}
Consider the ODE \eqref{eq:2D_Example}. The following statements hold.
\begin{itemize}
    \item[$i)$] For any $\mathbf{y}_0 \in \mathbb{R}^2$, there exist a \emph{period} $P\in \mathbb{R}$ such that
    \begin{equation*}
    \mathbf{s}(P+t,0,\mathbf{y}_0) = \mathbf{s}(t,0,\mathbf{y}_0), \quad \forall t \in \mathbb{R}\,.
\end{equation*}
\item[$ii)$] Let $P_{\gamma,\beta}$ denote the period of the ODE \eqref{eq:2D_Example} initialized at $\mathbf{y}(0) = \mathbf{y}_0 + \gamma\bm{\beta}$, where $\gamma>0$ and $\bm{\beta} \in \mathbb{R}^2$. The period increases or decreases as the Hamiltonian energy $H(\mathbf{y}(0))$ increases or decreases, respectively. More precisely
\begin{equation*}
    \text{sign}\left(P_{\gamma,\beta}-P_{0,\beta}\right) = \text{sign}\left(H(\mathbf{y}_0+\gamma \bm{\beta})-H(\mathbf{y}_0)\right)\,.
\end{equation*}
\end{itemize}
\end{lemma}

\vspace{0.1cm}

Lemma~\ref{le:periodic} implies that, for arbitrarily small $\epsilon$ and $\gamma$, the original and perturbed trajectories of \eqref{eq:2D_Example}  are not synchronized in phase.
 Let us now reason backwards in time. As $T-t$ decreases from $T$ to $0$, assuming without loss of generality that $\bm{\beta}$ points towards a larger Hamiltonian sublevel set $H(\mathbf{y}_0+\gamma \bm{\beta})>H(\mathbf{y}_0)$,  the value 
\begin{equation*}
\|\mathbf{s}(T,T-t,\mathbf{y}_0+\gamma \bm{\beta})-\mathbf{s}(T,T-t,\mathbf{y}_0)\|\,,
\end{equation*}
increases, until reaching, for a long enough time, a value of at least
$$D=\text{diam}(\{\mathbf{y}|~ H(\mathbf{y}) = H(\mathbf{y}_0)\}),$$
where the $\text{diam}(\cdot)$ for the sublevel set $\{\mathbf{y}|H(\mathbf{y})=H(\mathbf{y}_0)\}$ is defined as
\begin{align*}
  \text{diam}(\cdot) =& \max_{\mathbf{x}_1, \mathbf{x}_2} \|\mathbf{x}_1-\mathbf{x}_2\| \\
  &  \text{s.t.} \quad H(\mathbf{x}_1)=H(\mathbf{x}_2)=H(\mathbf{y}_0).
\end{align*}
We can connect these observations to the norm of the continuous-time BSM \eqref{eq:norm_grad_CT} as follows.  For any $\bm{\beta} \in \mathbb{R}^2$ that points towards a higher sublevel set of the Hamiltonian and  for any $\gamma>0$, there exist $T$ and $\tau$ such that
\begin{align}
   &\norm{ \frac{\partial \mathbf{y}(T)}{\partial \mathbf{y}(T-\tau)} \bm{\beta}} =\nonumber\\
   &=\norm{\lim_{\gamma\rightarrow 0^+} \frac{\mathbf{s}(T,T-\tau,\mathbf{y}_0+\gamma \bm{\beta})-\mathbf{s}(T,T-\tau,\mathbf{y}_0)}{\gamma}}\nonumber\\
   &=\lim_{\gamma\rightarrow 0^+} \frac{\norm{\mathbf{s}(T,T-\tau,\mathbf{y}_0+\gamma \bm{\beta})-\mathbf{s}(T,T-\tau,\mathbf{y}_0)}}{\gamma}\label{eq:value_validation_numerical}\\
   &= \lim_{\gamma\rightarrow 0^+}\frac{D}{\gamma} \nonumber\,.
\end{align}
Since the reached value $D$ is \emph{independent} of how small $\gamma$ is, we deduce that $\norm{ \frac{\partial \mathbf{y}(T)}{\partial \mathbf{y}(T-\tau)} \bm{\beta}}$ may diverge to infinity. In Appendix~\ref{sec:num_validation}, we confirm the above argument through numerical simulation of the considered H-DNN. Specifically, for any fixed value of $\gamma>0$, $\norm{ \frac{\partial \mathbf{y}(T)}{\partial \mathbf{y}(T-t)}}$ reaches a maximum value of approximately $\frac{D}{\gamma}$, which tends to infinity as $T$ increases and $\gamma$ approaches $0$.

Last, notice that the provided example is valid for arbitrarily small $\epsilon>0$. Hence, even if $\mathbf{y}(t)$ varies arbitrarily slowly, exploding gradients can still occur. Furthermore, the growth of the gradients is independent of whether the weights are chosen in a time-invariant or time-varying fashion. 

\subsubsection{Upper-bounds for general H-DNNs}

In the sequel, we consider time-varying weights
$\bm{\theta}(t)$  and general H-DNNs along with their continuous-time ODEs. As showcased, we expect exploding gradients as the network depth increases to infinity. Despite this fact, it remains important to derive upper-bounds on the continuous-time BSM for a fixed network depth. Characterizing how gradients grow across layers can provide network design guidelines to keep them under control. We present the following result whose proof can be found in Appendix~\ref{sec:lem:upperbound}. 
\begin{proposition}\label{lem:upper_bound}
	Consider a general H-DNN as per \eqref{eq:ODE_H} with depth $T \in \mathbb{R}$. Then,
	\begin{equation}
	\label{eq:exponential_bound}
	\norm{\frac{\partial \mathbf{y}(T)}{\partial \mathbf{y}(T-t)}}_2 \leq \sqrt{n} \exp(QT) \,, \quad \forall t \in [0,T]\,,
	\end{equation}
	where
	$Q = S\sqrt{n}\,\max_{t \in [0,T]}\norm{\mathbf{K}(t)}^2_2 \norm{\mathbf{J}(t)}_2 $ and $S$ satisfies \eqref{eq:sigma_bound}.
\end{proposition}

Proposition \ref{lem:upper_bound} quantifies the phenomenon of exploding gradients that we have observed in the example \eqref{eq:2D_Example}. 
Specifically, for a general H-DNN in continuous-time with a fixed depth $T \in \mathbb{R}$,  \eqref{eq:exponential_bound} reveals that the term $\norm{\mathbf{K}(t)}^2_2 \norm{\mathbf{J}(t)}_2$ is crucial in keeping the gradients under control. 
The same is expected for discrete-time implementations.

This fact leads to the following observation. 
When implementing an H-DNN, it is beneficial to add the regularizer $R_\ell(\bm{\theta}_j) = \|\mathbf{K}_j\|_2+\|\mathbf{J}_j\|_2$ in \eqref{eq:minimization} to control, albeit indirectly, the magnitude of BSMs. 
We exploit this regularization technique for image classification with the MNIST dataset in Section~\ref{sec:numerical_experiments}.

\subsection{Distributed learning through H-DNNs}
\label{sub:multi_agent}

In this subsection, we consider utilizing sparse weight matrices to enable distributed implementations of H-DNNs. Sparsity structures in neural networks can also be used to encode prior information on relations among elements when learning graph data~\cite{zhou2020graph} or to perform distributed control tasks~\cite{gama2021distributed, yang2021communication}.

First, we introduce the necessary notation for sparsity structures and binary matrices. Second, we characterize how we can design the sparsity patterns of the weight matrices $\mathbf{K}(t)$ and $\mathbf{J}(t)$ at each layer to enforce that nodes perform forward and backward propagation while complying with a fixed communication graph. Last, we observe that several sparsity choices lead to the same communication graph. Here, we deal with the continuous-time case for consistency. We extend the result to the DNNs obtained through S-IE discretization in Section~\ref{sec:DiscreteTimeSparseHDNN}. We start by introducing the necessary notation for defining and manipulating sparsity structures.

\subsubsection{Notation}
 For a block-matrix $\mathbf{W} \in \mathbb{R}^{m \times n}$, where $m = \sum_{i=1}^M m_i$ and $n = \sum_{i=1}^N n_i$, we denote its block in position $(i,j)$ as $\mathbf{W}^{i,j} \in \mathbb{R}^{m_i \times n_i}$.
Sparsity structures of matrices can be conveniently represented by binary matrices, i.e. matrices with $0/1$ entries. We use $\{0,1\}^{m \times n}$ to denote the set of $m \times n$ binary matrices.  
Let $\mathbf{Y}, \hat{\mathbf{Y}} \in \{0,1\}^{m \times n}$ and $\mathbf{Z} \in \{0,1\}^{n \times p}$. 
Throughout the paper, we adopt the following conventions: $\mathbf{Y} + \hat{\mathbf{Y}} \in \{0,1\}^{m \times n}$ and $ \mathbf{XZ}\in \{0,1\}^{m \times p}$ are binary matrices having a $0$ entry in position $(i,j)$ if and only if $\mathbf{Y} + \hat{\mathbf{Y}}$ and $\mathbf{XZ}$ have a zero entry in position $(i,j)$, respectively.
We say $\mathbf{X} \leq \hat{\mathbf{X}}$ if and only if $\mathbf{X}(i,j)\leq \hat{\mathbf{X}}(i,j)\;\forall i,j$, where $\mathbf{X}(i,j)$ denotes the entry in position $(i,j)$.
Finally, let $\mathbf{K} \in \mathbb{R}^{m \times n}$ be a matrix divided into $M \times M$ blocks $K^{i,j}\in \mathbb{R}^{m_i \times n_j}$, where $m=\sum_{k=1}^M m_k$ and $n=\sum_{k=1}^M n_k$. Let $\mathbf{R} \in \{0,1\}^{M \times M}$. We write that
\begin{equation*}
    \mathbf{K} \in \text{BlkSprs}(\mathbf{R})\,,
\end{equation*}
if and only if $ \mathbf{R}(i,j) = 0 \implies \mathbf{K}^{i,j} = 0_{m_i \times n_j}$.

\subsubsection{Distributed setup}

Let us  consider a set of nodes $i = 1,\ldots,M$ who collaboratively minimize a global cost function through an H-DNN. More precisely, the feature vector $\mathbf{y} \in \mathbb{R}^{n}$ is split into several subvectors as per
\begin{equation*}
    \mathbf{y} = (\mathbf{y}^{[1]},\dots,\mathbf{y}^{[M]}), \quad \mathbf{y}^{[i]} \in \mathbb{R}^{n_{i}}, ~\forall i=1,\ldots,M\,,
\end{equation*}
where $n= \sum_{i=1}^{M}n_i$. Each node can only share intermediate computations with a subset of the other nodes, denoted as $\emph{neighbors},$ according to a graph with adjacency matrix $\mathbf{S} \in \{0,1\}^{M \times M}$. We indicate the set of neighbors of a node $i$ as $\mathcal{N}_i$ and assume that $i \in \mathcal{N}_i$ for every $i=1,\ldots,M$. The local forward propagation update of node $i$ for every node $i=1,\ldots, M$ must be computed \emph{locally}, that is, it holds 
\begin{equation}
\label{eq:multiagent_FW}
\dot{\mathbf{y}}^{[i]}(t)= {\bf g}^{[i]}\left(\{\mathbf{y}^{[k]}(t)\}_{k \in \mathcal{N}_i}, {\boldsymbol{\theta}}^{[i]}(t)\right), \quad \forall t \in [0,T]\,,
\end{equation}
where $\mathbf{g}^{[i]}$ denotes the forward update of node $i$ and $\bm{\theta}^{[i]}(t)$ indicates the set of weights of node $i$ at time $t$. Similarly, the backward propagation update of node $i$  for every node $i=1,\ldots, M$ must be expressed, for all $t\in [0,T]$, as
\begin{equation}
\label{eq:multiagent_BW}
\dot{\bm{\delta}}^{[i]}(T-t) =  {\bf h}^{[i]}\left(\{\mathbf{y}^{[k]}(T-t),\bm{\delta}^{[k]}(T-t)\}_{k \in \mathcal{N}_i}, {\boldsymbol{\theta}}^{[i]}(t)\right)\,,
\end{equation}
where $\mathbf{h}^{[i]}$ denotes the backward update of node $i$. Next, we establish necessary and sufficient structural conditions on the weights $\bm{\theta}(t)$ of an H-DNN  to achieve localized forward and backward propagations. The proof is reported in Appendix~\ref{sec:th:multi_agent}.

\begin{theorem}[Distributed H-DNNs - continuous-time] 
\label{th:multi_agent}
Consider a distributed H-DNN  whose feature sub-vector for node $i$ has dimension $n_i$, and $\sum_{i=1}^M n_i = n$. Assume that for each $t \in [0,T]$ we have
\begin{equation*}
    \mathbf{K}(t) \in \text{BlkSprs}(\mathbf{R}(t)),\quad \mathbf{J}(t) \in \text{BlkSprs}(\mathbf{T}(t))\,,
\end{equation*}
where $\mathbf{T}(t),\mathbf{R}(t) \in \{0,1\}^{M \times M}$ are such that $\mathbf{T}(t),\mathbf{R}(t) \geq I_M$ and $\mathbf{T}(t)$ is symmetric, as required by the skew-symmetricity of $\mathbf{J}(t)$. Then, the forward update \eqref{eq:ODE_H} and the backward update \eqref{eq:delta_CT} are computed solely based on exchanging information with the neighbors encoded in $\mathbf{S} \in \{0,1\}^{M \times M}$ if, for every $t \in [0,T]$,
\begin{equation}
\label{eq:sparsity_invariance}
    \mathbf{T}(t)\mathbf{R}^\top(t)\mathbf{R}(t) + \mathbf{R}^\top(t)\mathbf{R}(t)\mathbf{T}(t) \leq \mathbf{S} \,.
\end{equation}
\end{theorem}

\vspace{0.2cm}

For instance, one can describe the case of completely isolated nodes by selecting $\mathbf{S} = I_M$; in this case $\mathbf{T}(t) = I_M$ and $\mathbf{R}(t)=I_M$ are the only admissible choices according to \eqref{eq:sparsity_invariance}. Naturally, good global classification performance cannot be achieved unless nodes are  allowed to exchange information at least with a small subset of others. As we will also see in numerical examples, it might be convenient to allow nodes to communicate according to a graph that is \emph{connected}. We leave the selection of high-performing communication topologies for future work.

Observe that, for a desired $\mathbf{S}$, there may be multiple choices for  $\mathbf{T}(t)$ and $\mathbf{R}(t)$ that satisfy \eqref{eq:sparsity_invariance}. One trivial choice is $\mathbf{T}(t) = \mathbf{S}$ and $\mathbf{R}(t) = I_M$ at all times. Another simple choice is $\mathbf{T}(t) = I_M$ and  $\mathbf{R}(t)$ such that $\mathbf{R}^\top(t) \mathbf{R}(t)\leq \mathbf{S}$. In general, there are many other choices that all lead to distributed computations based on the desired communication graph, as we illustrate in the following example. 

\vspace{0.1cm}

\begin{example}
Consider a continuous-time H-DNN with $M=4$, $T \in \mathbb{R}$, and the switching weight sparsities $\mathbf{R}(s) = I_M$, $\mathbf{T}(s) = \mathbf{S}$, $\mathbf{T}(\tau) = I_M$,
{\footnotesize
\begin{alignat*}{3}
\mathbf{S} &= \begin{bmatrix}1&1&1&1\\1&1&1&0\\1&1&1&1\\1&0&1&1\end{bmatrix}\,,\quad \mathbf{R}(\tau) &&= \begin{bmatrix}1&1&1&0\\1&1&0&0\\1&0&1&1\\0&0&1&1\end{bmatrix}\,,\\
\mathbf{T}(t) &= \begin{bmatrix}1&0&1&0\\0&1&1&0\\1&1&1&1\\0&0&1&1\end{bmatrix}\,, \quad \mathbf{R}(t)&&= \begin{bmatrix}1&0&0&1\\1&1&0&0\\0&0&1&0\\0&0&0&1\end{bmatrix}\,,
\end{alignat*}
}for every $s \in [0,\frac{T}{3})$ $\tau \in [\frac{T}{3},\frac{2T}{3})$ and every $t \in [\frac{2T}{3},T]$. It is easy to verify that condition \eqref{eq:sparsity_invariance} is satisfied for all $t\in [0,T]$. Hence, the four nodes composing the H-DNN  can choose their non-zero local weights arbitrarily and compute forward and backward updates without any direct exchange of information between node $2$ and node $4$. 
\end{example}

\vspace{0.2cm}

In Section~\ref{sec:DT_analysis}, we see how the result of Theorem~\ref{th:multi_agent} allows for distributed H$_1$-DNN and H$_2$-DNN implementations, and in Section~\ref{sec:numerical_experiments}, we validate the effectiveness of distributed H-DNNs through numerical experiments.

\section{Discrete-time analysis}\label{sec:DT_analysis}

Having established useful properties of H-DNNs from a continuous-time perspective, our next goal is to preserve such properties after discretization. In this section, we achieve this goal by using the  S-IE discretization method. 
First, we analyze the symplecticity of $\frac{\partial {\bf y}_{l+1}}{\partial {\bf y}_l}$. 
Second, we formally prove that the BSMs in discrete-time needed for backpropagation are lower-bounded in norm by the value $1$, independent of the network depth and of the choice for the time-varying weights.
Finally, we extend the analysis to the distributed learning setup in discrete-time.

\subsection{Non-vanishing gradients for H$_2$-DNNs}\label{subs:non-vanishing}

In Section~\ref{sec:symplecticity_CT}, we analyzed  the symplecticity of  $\frac{\partial {\bf y}(T)}{\partial {\bf y}(T-t)}$. 
In this section, we show that, under mild conditions, S-IE discretization preserves the symplectic property \eqref{eq:generalized_symplecticity} for the BSM. 
In turn, this allows us to show that gradients cannot vanish for any H-DNN architecture based on S-IE discretization.
We analyze the symplecticity of $\frac{\partial {\bf y}_{l+1}}{\partial {\bf y}_l}$ with respect to $\mathbf{J}$ as per Definition~\ref{def:symplectic}.

\begin{lemma}\label{lem:numerical_flow}
	Consider the time-varying system \eqref{eq:implicit_euler} and assume that the time-invariant matrix ${\bf J}$ has the block structure in \eqref{eq:implicit_euler_JKb}. Then, one has
	\begin{equation}\label{eq:numerical_symplecticity_y}
	\left[\frac{\partial {\bf y}_{l+1}}{\partial {\bf y}_{l}}\right]^\top {\bf J} \left[\frac{\partial {\bf y}_{l+1}}{\partial {\bf y}_{l}}\right] = {\bf J}\,.
	\end{equation}
    for all $l=1,\dots,N-1$.
\end{lemma}

The proof of Lemma~\ref{lem:numerical_flow} can be found in Appendix~\ref{ap:lem_numerical_flow} and  is built upon the result of Theorem 3.3 of Section VI in \cite{Hairer2006book} and the definition of extended Hamiltonian systems \cite{deGosson2011book}. 
Theorem 3.3 in \cite{Hairer2006book} proves that the numerical flow of a \textit{time-invariant} Hamiltonian system with ${\bf J} = \begin{bsmallmatrix}0 & I \\ -I & 0 \end{bsmallmatrix}$ is symplectic. 
Moreover, by defining an extended Hamiltonian system, we can embed the study of a time-dependent Hamiltonian function into the time-independent case by defining an extended phase space of dimension $n+2$ instead of $n$.

Lemma~\ref{lem:numerical_flow} allows us to prove that the BSMs of H$_2$-DNNs are lower-bounded in norm by $1$ irrespective of the parameters and the depth of the network.  

\begin{theorem}\label{th:nonvanishing_gradients_DT}
\label{th:nonvanishing_gradients_numerical}
Consider the H$_2$-DNN in \eqref{eq:implicit_euler}. Assume that ${\bf J}$ has the block structure in \eqref{eq:implicit_euler_JKb} and $\mathbf{J}\neq 0_n$.
	Then,
    \begin{equation}
    \label{eq:nonvanishing_gradients_DT}
        \norm{\frac{\partial {\bf y}_N}{\partial {\bf y}_{N-j}}} \geq 1\,,
    \end{equation}
     for all $j = 0,\dots,N-1$, where $\norm{\cdot}$ denotes any sub-multiplicative norm.
\end{theorem}
\begin{proof}
We know by Lemma~\ref{lem:numerical_flow} that \eqref{eq:numerical_symplecticity_y} holds. 
Moreover,
\begin{equation}\label{eq:corol2_aux1}
    \frac{\partial {\bf y}_{l+2}}{\partial {\bf y}_{l}} = \frac{\partial {\bf y}_{l+1}}{\partial {\bf y}_{l}} \frac{\partial {\bf y}_{l+2}}{\partial {\bf y}_{l+1}}\,.
\end{equation}
Then, by developing the terms in $\frac{\partial {\bf y}_{N}}{\partial {\bf y}_{N-j}}$ as per \eqref{eq:corol2_aux1}, calculating its transpose when needed, and applying iteratively \eqref{eq:numerical_symplecticity_y}, we obtain
\begin{equation*}
	\left[\frac{\partial {\bf y}_{N}}{\partial {\bf y}_{N-j}}\right]^\top {\bf J} \left[\frac{\partial {\bf y}_{N}}{\partial {\bf y}_{N-j}}\right] = {\bf J}\,.
\end{equation*}
Hence, we have 
\begin{align*}
     \|\mathbf{J}\| &= \norm{\left(\frac{\partial {\bf y}_{N}}{\partial {\bf y}_{N-j}}\right)^\top \mathbf{J} \left(\frac{\partial {\bf y}_{N}}{\partial {\bf y}_{N-j}}\right)}
     \leq \norm{\frac{\partial {\bf y}_N}{\partial {\bf y}_{N-j}}}^2 \norm{\mathbf{J}}\,,
 \end{align*}
for all $j = 0,\dots,N-1$. This inequality implies~\eqref{eq:nonvanishing_gradients_DT}.
\end{proof}

\subsection{Distributed H$_2$-DNN implementation}\label{sec:DiscreteTimeSparseHDNN}

We show how to port the distributed setup developed in Section~\ref{sub:multi_agent} to H$_2$-DNNs. Moreover, we analyze how the choice of the sparsity structures impacts the weights $\bm{\theta}^\star$ verifying \eqref{eq:stationary_points}.
 
\begin{proposition}[Distributed H$_2$-DNN]
\label{prop:distributed_implicit_Euler}
Consider the system \eqref{eq:ODE_H} and assume that \eqref{eq:implicit_euler_JKb} holds. The S-IE discretization when considering $\mathbf{y}_j = (\mathbf{p}_j,\mathbf{q}_j)$ gives the forward propagation updates \eqref{eq:p_update}-\eqref{eq:q_update} and the backward propagation updates \eqref{eq:backward_implicit_euler}.
Moreover, consider nodes $i =1,\dots,M$ whose feature sub-vectors are denoted as $\mathbf{y}_j^{[i]}=(\mathbf{p}_j^{[i]},\mathbf{q}^{[i]}_j)$ at every layer $j$.  Assume that at each layer we impose
\begin{equation*}
    \mathbf{K}_{p,j}, \mathbf{K}_{q,j}\in \text{BlkSprs}(\mathbf{R}_j),\quad \mathbf{X} \in \text{BlkSprs}(\mathbf{T})\,,
\end{equation*}
where $\mathbf{T},\mathbf{R}_j \in \{0,1\}^{M \times M}$ are such that $\mathbf{T},\mathbf{R}_j \geq I_M$ and $\mathbf{T}$ is symmetric. Then, the following two facts hold.
\begin{itemize}
    \item[$i)$] The forward update \eqref{eq:p_update}-\eqref{eq:q_update} and the backward update \eqref{eq:backward_implicit_euler} are computed solely based on exchanging information with the neighbors encoded in $\mathbf{S} \in \{0,1\}^{M \times M}$ if, for every $j =0,\dots,N-1$,
\begin{equation}
\label{eq:sparsity_invariance_IE}
    \mathbf{T}\mathbf{R}^\top_j\mathbf{R}_j + \mathbf{R}^\top_j\mathbf{R}_j\mathbf{T} \leq \mathbf{S} \,.
\end{equation}
\item[$ii)$] The training of this H-DNN asymptotically steers the weights $\bm{\theta}_j$ towards a $\bm{\theta}^\star_j$ such that for all $j =0,\dots,N-1$
\begin{align*}
    &\nabla_{\mathbf{K}_{p,j}}\tilde{\mathcal{L}}(\bm{\theta}^\star_j) \odot \mathbf{R}_j = \nabla_{\mathbf{K}_{q,j}}\tilde{\mathcal{L}}(\bm{\theta}^\star_j)\odot \mathbf{R}_j = 0_{n \times n} \,,\\
    &\nabla_{\mathbf{X}}\tilde{\mathcal{L}}(\bm{\theta}^\star_j)\odot \mathbf{T} = 0_{n \times n}\,, \quad \nabla_{\mathbf{b}_j}\tilde{\mathcal{L}}(\bm{\theta}^\star_j) = 0_{n \times 1}\,,
\end{align*}
where $\odot$ denotes the Hadamard product and $\tilde{\mathcal{L}} = \mathcal{L}+R$.
\end{itemize}

 \end{proposition}
\begin{proof}
The proof of point $i)$ is analogous to that of Theorem~\ref{th:multi_agent} by splitting the weight matrix $\mathbf{K}$ into submatrices $\mathbf{K}_p$ and $\mathbf{K}_q$. For point $ii)$, the regularized loss $\tilde{\mathcal{L}}$ can be rewritten as a function of only those weight entries that are not constrained to be equal to $0$. The neural network training will thus steer these non-zero entries to a point where the gradient value is zero along the corresponding components. However, nothing can be said about the gradient value along the components such that $\mathbf{R}_j(i,l) = 0$ or $\mathbf{T}(k,v) = 0$, as the corresponding weights cannot be selected to further decrease the loss function.
\end{proof}

 The second point of Proposition~\ref{prop:distributed_implicit_Euler} shows that, in a distributed H-DNN complying with a given structure, the achieved stationary points may be very different depending on the specific choice of weight sparsities complying with \eqref{eq:sparsity_invariance_IE}. Hence, the achieved performance will also strongly depend on this choice. Simple algorithms to efficiently determine optimized choices for $\mathbf{T}$ and $\mathbf{R}_j$ were proposed in the context of structured feedback control; we refer the interested reader to the concept of sparsity invariance  \cite{furieri2019separable,furieri2020sparsity}. 

To conclude, we note that a distributed H$_1$-DNN based on FE discretization can also be obtained in a straightforward way. Indeed,  the FE update equations (both forward \eqref{eq:H-DNN_fE} and backward \eqref{eq:FE_BP}) preserve the same structure as their continuous-time counterparts (\eqref{eq:ODE_H} and \eqref{eq:delta_CT}).
The performance of distributed H-DNNs is tested in Section~\ref{sec:numerical_experiments}.

\section{Numerical experiments}\label{sec:numerical_experiments}

In this section, we demonstrate the potential of H-DNNs on various classification benchmarks, including the MNIST dataset.\footnote{The code is available at: \url{https://github.com/DecodEPFL/HamiltonianNet}.}
Our first goal is to show that H-DNNs are expressive enough to achieve state-of-art performance on these tasks, even when distributed architectures are used. 
Then, we validate our main theoretical results by showing that gradients do not vanish despite considering deep H$_2$-DNNs. Instead, when using the same data, standard multilayer perceptron networks do suffer from this problem, which causes early termination of the learning process.

Further numerical experiments showing how H-DNNs can be embedded in complex architectures for image classification can be found in Appendix~\ref{ap:CIFAR}. Specifically, 
we test an enhanced architecture of H-DNNs over the CIFAR-10 dataset and we show comparable performance with state-of-the-art architectures.

\subsection{Binary classification examples}

We consider two benchmark classification problems from \cite{Haber_2017} with two categories and two features. 
First, we show that H$_1$- and H$_2$-DNNs perform as well as the networks MS$_i$-DNNs, $i=1,2,3$ given in Appendix~\ref{ap:relation_existing_arq} and introduced in \cite{Haber_2017, Chang19} and \cite{Chang18a}.
Second, we test a distributed H$_2$-DNN architecture, showing that it can achieve excellent accuracy despite local computations.
In all cases, we optimize over the weights ${\bf K}_j$ and ${\bf b}_j$ and set 
\begin{equation}
    {\bf J}_j = {\bf J} =  \begin{bmatrix} 0 & -I \\ I &0 \end{bmatrix}\,,
    \label{eq:J_simulations}
\end{equation}
for $j=0,\dots,N-1$.

As outlined in Section \ref{sec:cost}, we complement the DNNs with an output layer ${\bf y}_{N+1} = {\bf f}_N({\bf y}_N, \boldsymbol{\theta}_N)$ consisting of a sigmoid activation function applied over a linear transformation of ${\bf y}_N$.
In \eqref{eq:minimization}, we use a standard binary cross-entropy loss function $\mathcal{L}$ \cite{GoodBengCour2016}.
The optimization problem is solved by using the Adam algorithm; we refer to Appendix~\ref{ap:implementation_2d} for details on the implementation.

\subsubsection{Comparison with existing networks}

We consider the ``Swiss roll'' and the ``Double moons'' datasets shown in Figure~\ref{fig:examplesClassification}. By performing feature augmentation \cite{dupont2019augmented}, we use input feature vectors given by 
$({\bf y}_0^k, 0_{2\times 1}) \in \mathbb{R}^4$ 
where 
${\bf y}_0^k \in \mathbb{R}^2, k=1,\dots,s$ are the input datapoints.

In Table~\ref{tab:classif_L4DC}, we present the classification accuracies over test sets when using MS- and H-DNNs with different number of layers. We also give the number of parameters per layer of each network. 
It can be seen that, for a fixed number of layers, the performances of H$_1$-DNN and H$_2$-DNN are similar or better than the other networks. 
This can be motivated by the fact that the new architectures are more expressive than MS$_i$-DNNs, since they have, in general, more parameters per layer. Although not  shown, H- and MS-DNNs with the same number of parameters have similar performance.

\begin{table}
	\begin{center}
    \caption{Classification accuracies over test sets for different examples using different network structures with ${\bf y}_j \in \mathbb{R}^4$  for each layer $j$. The first three columns represent existing architectures (MS$_i$-DNN) while the two last columns provide the results for the new H$_i$-DNNs. The first two best accuracies in each row are in bold. The last row provides the number of parameters per layer of each network.}
    \label{tab:classif_L4DC}
    \begin{tabular}{r|c|c|c|c||c|c}
        & & MS$_1$- & MS$_2$- & MS$_3$- & H$_1$- & H$_2$- \\
        \hline
        Swiss & 4 layers & 77.1\% & 79.7\% & \textcolor{dark-gray}{\textbf{90.1\%}} & \textbf{93.6\%} & 84.3\% \\
        roll & 8 layers & 91.5 \% & 90.7\% & 87.0\% & \textbf{99.0\%} & \textcolor{dark-gray}{\textbf{95.5\%}} \\
        & 16 layers & 97.7\% & 99.7\% & 97.1\% & \textcolor{dark-gray}{\textbf{99.8\%}} &  \textbf{100\%} \\
        & 32 layers & \textbf{100\%} & \textbf{100\%} & 98.4\% & \textcolor{dark-gray}{\textbf{99.8\%}} & \textbf{100\%} \\ 
        & 64 layers & \textbf{100\%} & \textbf{100\%} & \textbf{100\%} & \textcolor{dark-gray}{\textbf{99.8\%}} & \textbf{100\%} \\
        \hline
        Double & 1 layer & 92.5\% & 91.3\% & \textcolor{dark-gray}{\textbf{97.6\%}} & \textbf{100\%} &  94.4\% \\ 
        moons & 2 layers & 98.2\% & 94.9\% & \textcolor{dark-gray}{\textbf{99.8\%}} & \textbf{100\%} &  \textcolor{dark-gray}{\textbf{99.8\%}} \\
        & 4 layers & \textcolor{dark-gray}{\textbf{99.5\%}} & \textbf{100\%} & \textbf{100\%} & \textbf{100\%} & \textbf{100\%} \\
        \hline
        \hline
        \multicolumn{2}{c|}{\textit{\# param. per layer}} & $8$ & $10$ & $12$ & $20$ & $12$\\
        \hline
    \end{tabular}
	\end{center}
\end{table}

\begin{figure}
	\centering
	\begin{subfigure}{.48\linewidth}
		\centering
		\includegraphics[width=\linewidth]{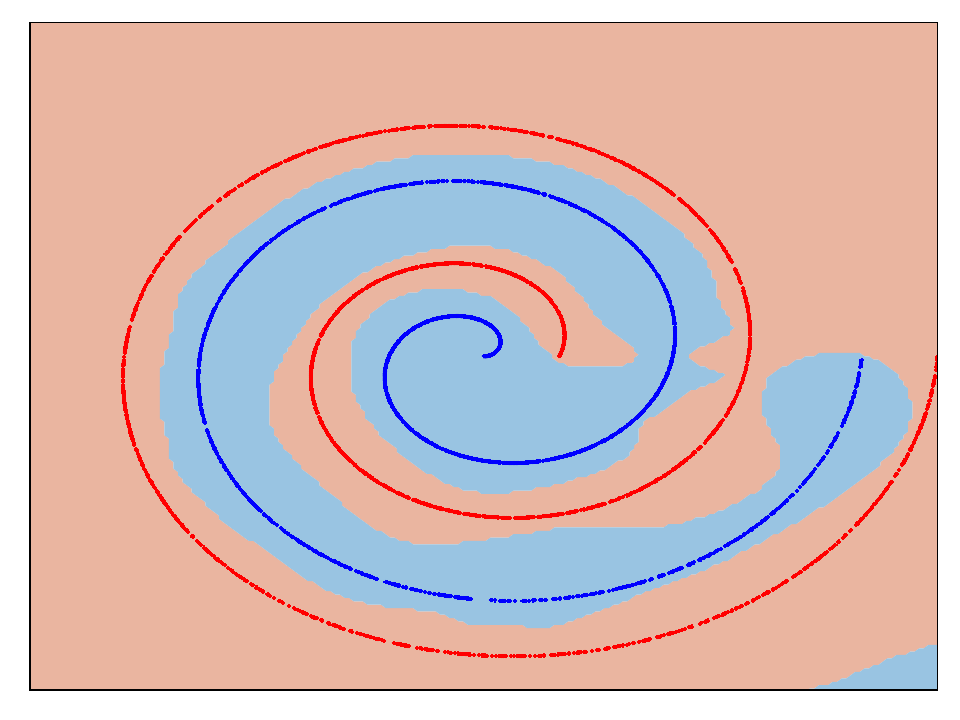}  
		% T=1.2, 80 iter // 100% accuracy over test set
		\caption{``Swiss roll'' - 64 layers}
		\label{fig:examplesClassification_swissroll}
	\end{subfigure}%
	\begin{subfigure}{.48\linewidth}
		\centering
		\includegraphics[width=\linewidth]{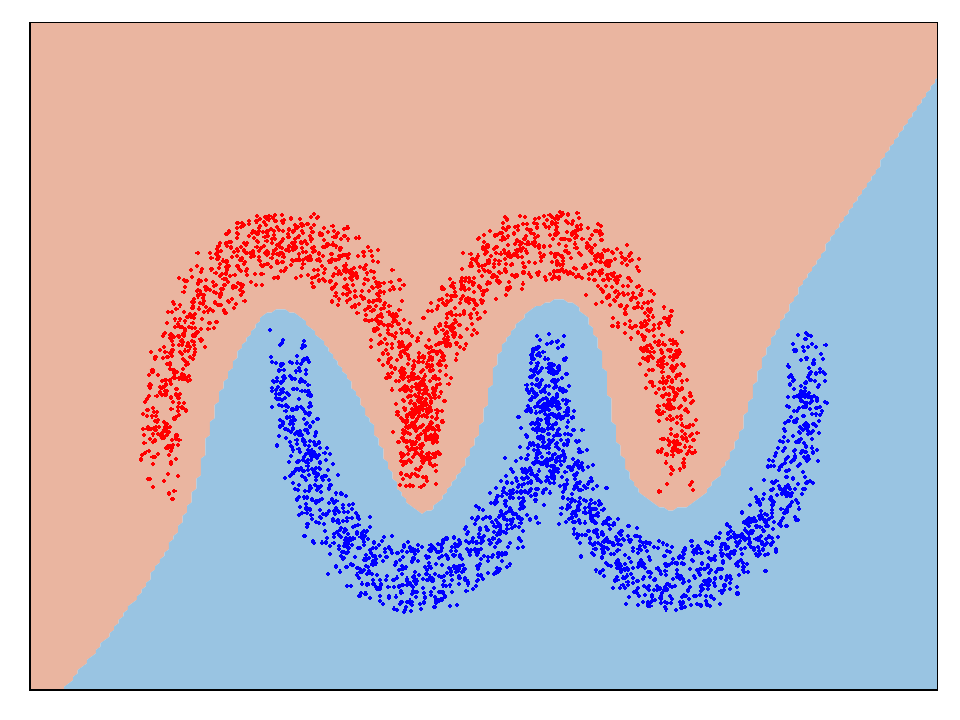}  
		% T=0.8 // 100% accuracy over test set
		\caption{``Double moons'' - 4 layers}
		\label{fig:examplesClassification_doublemoons}
	\end{subfigure}%
	\caption{Results for the H$_1$-DNN architecture. Colored regions represent the predictions of the trained DNNs.}
	\label{fig:examplesClassification}
\end{figure}

\subsubsection{Distributed vs. centralized training}\label{subsubsub:example_distributed}

We test the effectiveness of H$_2$-DNNs for a distributed learning task under constraints \eqref{eq:implicit_euler_JKb}.
We assume that $8$ nodes have access to their own local features ${\bf y}^{[i]} \in \mathbb{R}^2, i= 1,\dots,8$.
Moreover, they communicate with a small set of neighbors ${\bf y}^{[j]}$  according to the connected graph in Figure~\ref{fig:interconnection}, whose adjacency matrix is denoted with ${\bf S}$.
Therefore, each node computes local forward and backward updates only based on the features known to itself and to its first- and second-order neighbors.

In order to comply with the communication constraints, we exploit \eqref{eq:sparsity_invariance_IE} in Proposition~\ref{prop:distributed_implicit_Euler}, which is verified, for $j=0,
\ldots,N$, by
 ${\bf T}=I_8$ and ${\bf R}_j$ equal to the adjacency matrix of the subgraph including only the blue edges in Figure~\ref{fig:interconnection}.
Since ${\bf J}_j$ in the form of \eqref{eq:implicit_euler_JKb} is chosen equal to \eqref{eq:J_simulations}, then  ${\bf X} \in \text{BlkSprs}({\bf T})$.
Accordingly, the weights ${\bf K}_{p,j}$, ${\bf K}_{q,j}$ are constrained to lie in $\text{BlkSprs}({\bf R}_j)$. 

Then, we initialize the augmented input feature vector $\tilde{\bf y}_0 = ({\bf p}_0, {\bf q}_0) \in \mathbb{R}^{16}$ with zeros, except for the two entries $\mathbf{p}_0^{[1]}= y^k_{0,0}$ and $\mathbf{q}_0^{[5]}= y^k_{0,1}$. 

\begin{figure}
	\centering
	\includegraphics[width=0.4\linewidth]{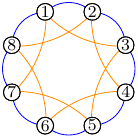}
	\caption{Communication network for the distributed H$_2$-DNN in Section~\ref{subsubsub:example_distributed}. Blue and orange edges connect each node with its first-order and second-order neighbors, respectively.}
	\label{fig:interconnection}
\end{figure}

Table~\ref{tab:classif_sparse} reports the prediction accuracies for the distributed and centralized learning settings when using the ``Swiss roll'' and ``Double circles'' datasets shown in Figure~\ref{fig:doublecircles_swissroll}. Moreover, the colored regions in Figure \ref{fig:doublecircles_swissroll} provide an illustration of the predictive power of the distributed H$_2$-DNNs with four layers. From Table~\ref{tab:classif_sparse}, 
even if the centralized architectures perform better for shallow networks (e.g. 2 layers), it is worth noticing that the distributed architectures use less than half of the parameters. 
Moreover, if the networks are deep enough, the classification task is successfully solved in both settings.

\begin{table}
	\begin{center}
    \caption{Classification accuracy over test sets for two benchmark examples by using centralized and distributed H$_2$-DNNs (the latter using the communication network in Fig.~\ref{fig:interconnection}).}
    \label{tab:classif_sparse}
    \small
    \begin{tabular}{r|cc|cc}
        & \multicolumn{2}{c|}{Swiss roll} & \multicolumn{2}{c}{Double circles} \\
        \cline{2-5}
        \textit{\# layers} & Distributed & Centralized & Distributed & Centralized \\
        \hline
        2 & 91.05\% & 99.90\% & 98.65\% & 100\%  \\ 
        3 & 99.08\% & 100\% & 99.42\% & 99.85\% \\
        4 & 100\% & 100\% & 99.70\% & 100\% \\
        \hline
        \textit{\# param.} & \multirow{2}{*}{112} & \multirow{2}{*}{272} & \multirow{2}{*}{112} & \multirow{2}{*}{272} \\
        \textit{per layer} & & & & \\
    \end{tabular}
	\end{center}
\end{table}

\begin{figure}
	\begin{subfigure}{0.49\linewidth}
	\centering
		\includegraphics[width=\linewidth]{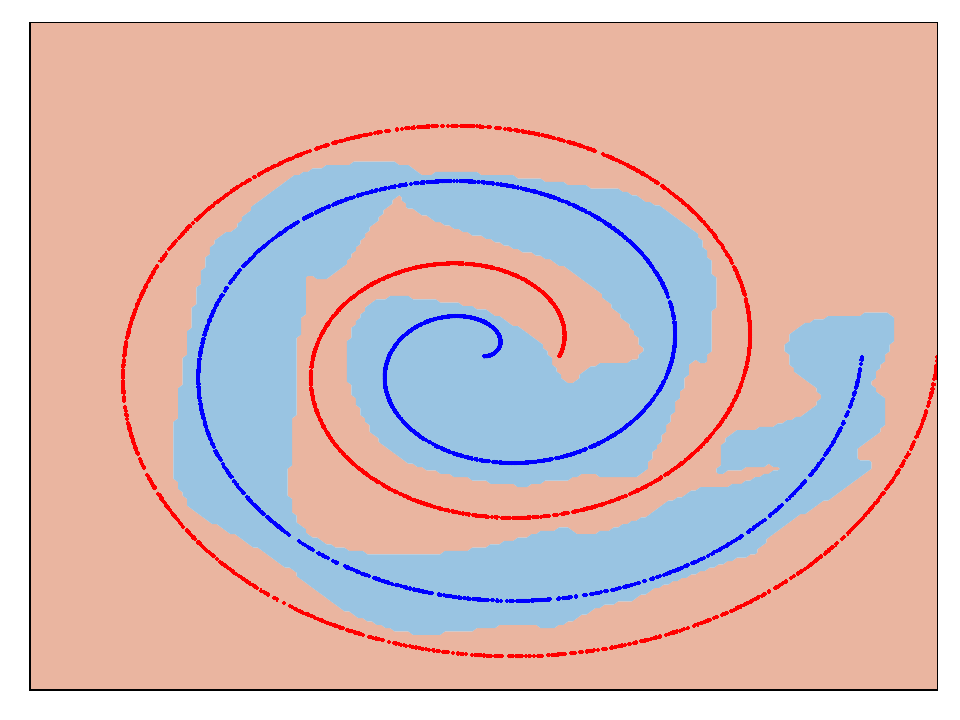}
		\caption{``Swiss roll'' dataset}
	\end{subfigure}
	\begin{subfigure}{0.49\linewidth}
	\centering
		\includegraphics[width=\linewidth]{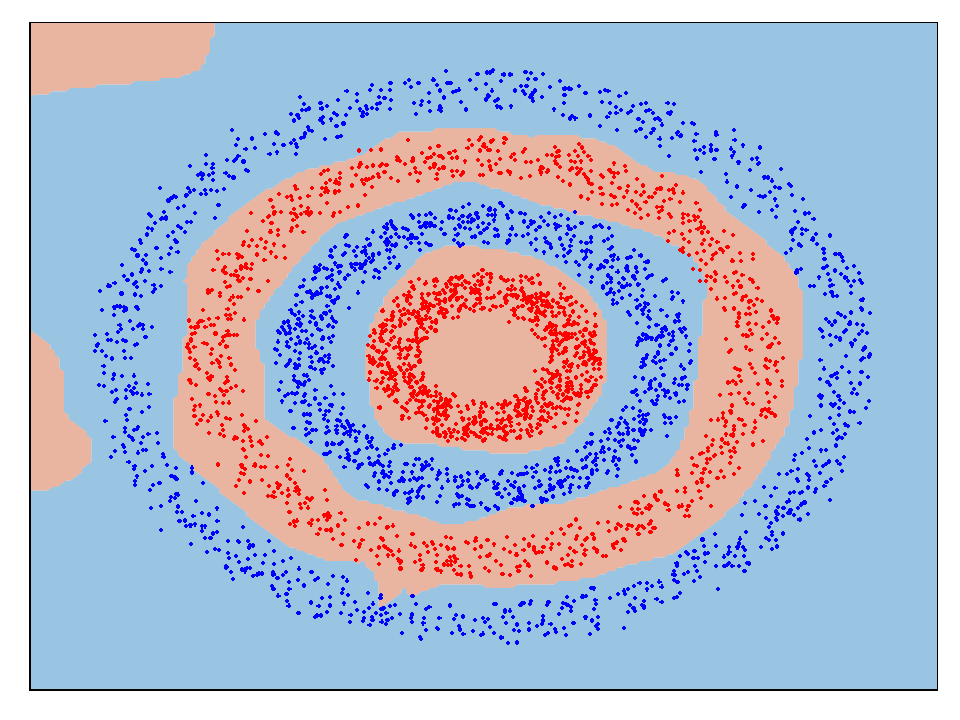}
		\caption{``Double circles'' dataset}
	\end{subfigure}%
	\caption{Results for the distributed learning with 4-layer H$_2$-DNNs. Colored regions represent the predictions of the trained DNNs.}
	\label{fig:doublecircles_swissroll}
\end{figure}

\subsection{Experiments with the MNIST dataset}\label{sec:mnist}

We evaluate our methods on a more complex example: the image classification benchmark MNIST.\footnote{\url{http://yann.lecun.com/exdb/mnist/}}
The dataset consists of $28 \times 28$ digital images in grayscale of hand-written digits from 0 to 9 with their corresponding labels. It contains 60,000 training examples and 10,000 test examples.

Following \cite{Haber_2017}, we use a network architecture consisting of a convolutional layer followed by an H-DNN and an output layer with soft-max activation function. 
The convolutional layer is a linear transformation that expands the data from 1 to 8 channels, and 
the network output is a vector in $\mathbb{R}^{10}$ representing the probabilities of an image belonging to each of the 10 classes.

We compare the performance of MS$_1$-DNNs and H$_1$-DNNs\footnote{Similar results can be obtained using other MS- or H-DNNs.} with 2, 4, 6 and 8 layers. 
We set $h = 0.4$ for MS$_1$-DNNs and $h = 0.5$ for H$_1$-DNNs. The implementation details can be found in Appendix~\ref{ap:implementation_mnist}.

In Table~\ref{tab:mnist}, we summarize the train and test accuracies of the different DNN architectures. 
The first row of the table provides, as a baseline, the results obtained when omitting the Hamiltonian DNN block, i.e., when using only a convolutional layer followed by the output layer. We observe that both MS$_1$-DNN and H$_1$-DNN achieve similar performance.
Note that, while the training errors are almost zero, the test errors are reduced when increasing the number of layers, hence showing the benefit of using deeper networks.
Moreover, these results are in line with test accuracies obtained when using standard convolutional layers or ResNets instead of H-DNNs~\cite{Chang18a,Haber_2017}.

\begin{table}
	\small
	\begin{center}
    \caption{Classification accuracies over training and test sets for the MNIST example when using MS$_1$-DNN and  H$_1$-DNN architectures. A convolutional layer and an output layer are added before and after each DNN. The first row, corresponding to 0 layers, refers to a network with a single convolutional layer followed by an output layer. 
    }
    \label{tab:mnist}
    \begin{tabular}{|c|c c|c c|}
        \hline
        Number of& \multicolumn{2}{c}{MS$_1$-DNN} & \multicolumn{2}{|c|}{H$_1$-DNN} \\ \cline{2-5} 
        layers & Train & Test & Train & Test \\ \hline
        \hline
        0 &  93.51\% & 92.64\% & - & - \\ 
        2 & 99.20\% & 97.95\% & 99.00\% & 98.04\% \\   
        4 & 99.11\% & 98.23\% & 99.32\% & 98.28\% \\  
        6 &  99.58\% & 98.10\% & 99.41\% & 98.37\% \\ 
        8 & 99.80\% & 98.26\% & 99.39\% & 98.38\% \\ 
        \hline
    \end{tabular}
	\end{center}
\end{table}

\subsection{Gradient analysis}

\begin{figure*}[t]
	\begin{subfigure}{0.5\linewidth}
		\centering
		\includegraphics[width=\linewidth]{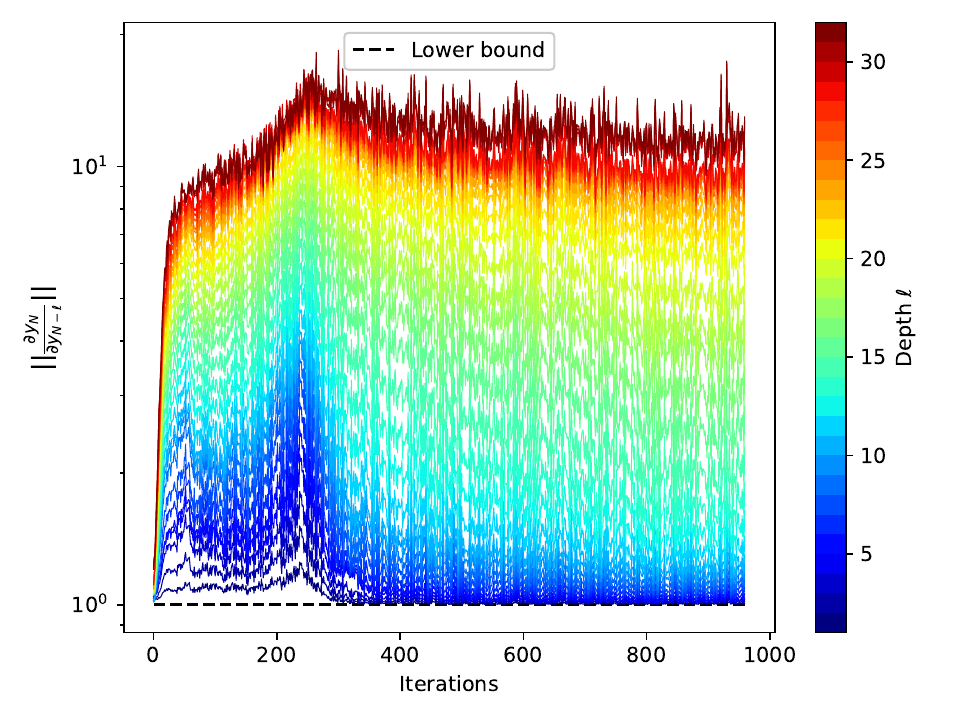}
		\caption{}
		\label{fig:gradients_a}
	\end{subfigure}%
	\begin{subfigure}{0.5\linewidth}
		\centering
		\begin{minipage}{\linewidth}
			\centering
			\includegraphics[width=\linewidth]{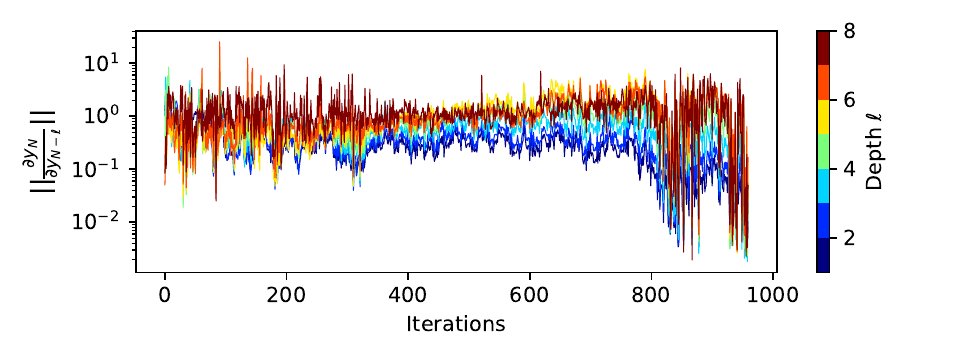}
		\end{minipage}
	
		\begin{minipage}{\linewidth}
			\centering
			\includegraphics[width=\linewidth]{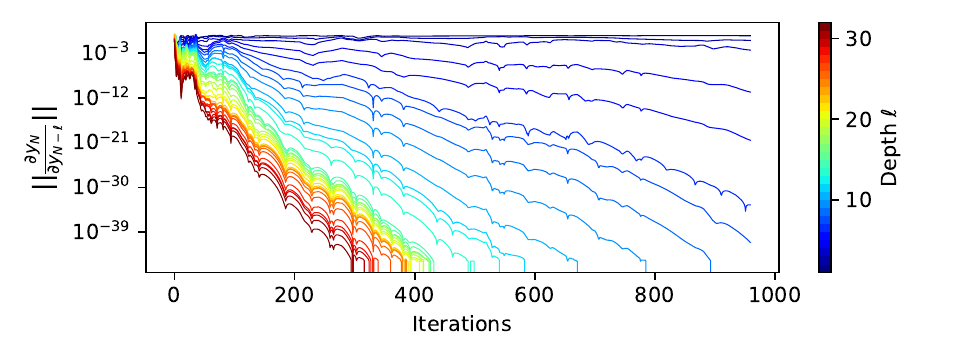}
		\end{minipage}%
	\caption{}
	\label{fig:gradients_b}
	\end{subfigure}%
	\caption{Evolution of the 2-norm of the BSM 
	during the training of (a) a 32-layer H$_2$-DNN and (b) a multilayer perceptron network with 8 (above) and 32 (below) layers.}
	\label{fig:gradients}
\end{figure*}

Our next aim is to provide a numerical validation of the main property of H-DNNs: the absence of vanishing gradients.
We consider the ``Double moons" dataset shown in Figure~\ref{fig:examplesClassification_doublemoons} and analyze the norm of BSMs during the training of an H$_2$-DNN and a fully connected MLP network.\footnote{The layer equation of the MLP network is given by ${\bf y}_{j+1} = \sigma({\bf K}_j {\bf y}_j + {\bf b}_j)$, with activation function $\sigma(\cdot) = \tanh(\cdot)$ and trainable parameters ${\bf K}_j$ and ${\bf b}_j$. We use the same implementation as described in Appendix~\ref{ap:implementation_2d}.}

Figure~\ref{fig:gradients_a} displays the BSM norms during the 960 iterations of the training phase of a 32-layer H$_2$-DNN and Figure~\ref{fig:gradients_b} presents the same quantities for an MLP network with 8 and 32 layers. 
While the H$_2$-DNN and the 8-layer MLP network achieve good performance at the end of the training ($100\%$ and $99.7\%$ accuracy over the test set, respectively), the 32-layer MLP network fails to classify the input data ($50\%$ accuracy). 

Figure~\ref{fig:gradients_a} validates the results of Theorem~\ref{th:nonvanishing_gradients_DT} since no BSM norm is smaller than 1 at any time.\footnote{We highlight that 
$\mathcal{L}(\boldsymbol{\theta})$ 
achieves a stationary point $\boldsymbol{\theta}^\star$ 
since 
$\nabla_{\boldsymbol{\theta}}\mathcal{L}(\boldsymbol{\theta}^\star) = 0$
in approximately 350 iterations,
obtaining a $100\%$ accuracy over the test set.}
However, this is not the case for MLP networks where gradient norms can be very small and may prevent the network to succeed in the training.
Note that the main inconvenient for the 32-layer MLP network to achieve good performance is that after 400 iterations, only a few layers still show a gradient norm different from zero i.e. $\norm{\frac{\partial {\bf y}_N}{\partial {\bf y}_{N-\ell}}} \approx 0$ for $\ell=7,\dots, 32$.

In addition, it can be seen from Figure~\ref{fig:gradients_a} that the BSM norms do not explode and remain bounded as
\begin{equation*}
    1\leq \norm{\frac{\partial \mathbf{y}_N}{\partial \mathbf{y}_{N-\ell}}} \leq 11\,, \quad \forall \ell =0,\dots,N-1\,.
\end{equation*}
This is in line with Proposition~\ref{lem:upper_bound}, where we show that, at each iteration of the training phase, BSM norms are upper-bounded by a quantity depending on the network depth and the parameters $\boldsymbol{\theta}_{0,\dots,N-1}$ of that specific iteration.

\section{Conclusion}\label{sec:Conclusion}

In this paper, we proposed a class of H-DNNs obtained from the time discretization of Hamiltonian dynamics.
We proved that H-DNNs stemming from S-IE discretization do not suffer from vanishing gradients and also provided methods to control the growth of associated BSMs.
These results are obtained by combining concepts from system theory, as per Hamiltonian systems modeling, and discretization methods developed in numerical analysis.
We further derived sufficient structural conditions on weight matrices to enable distributed learning of H-DNNs.
Although we limited our analysis to S-IE discretization, one can leverage the rich literature on symplectic integration~\cite{Hairer2006book} for defining even broader classes of H-DNNs. This avenue will be explored in future research. %Moreover, we will study the use of distributed H-DNNs architectures for parametrizing distributed controllers for multi-agent systems.
It is also relevant to study the application of H-DNNs to optimal control problems — we refer the interested reader to~\cite{Furieri22a} for preliminary results in this direction.

\bibliographystyle{IEEEtran}
\bibliography{references}

\appendices
\section{Relationship of H-DNNs with existing architectures}
\label{ap:relation_existing_arq}

We introduce existing architectures proposed in \cite{Haber_2017, Chang19} and \cite{Chang18a}, and show how they can be encompassed in our framework. Since, for constant weights across layers, these architectures stem from the discretization of marginally stable systems, we call them  MS$_i$-DNN ($i=1,2,3$).

% MS1 DNN
\noindent\textbf{MS$_1$-DNN:} In \cite{Haber_2017}, the authors propose to use Verlet integration method to discretize
$$
\begin{bmatrix}
\dot{\bf p}(t) \\
\dot{\bf q}(t)
\end{bmatrix}
\hspace{-0.1cm}
=
\sigma \left(
\begin{bmatrix}
0_\frac{n}{2} & {\bf K}_0(t)\\
-{\bf K}_0^\top(t) & 0_\frac{n}{2}
\end{bmatrix}
\begin{bmatrix}
{\bf p}(t) \\
{\bf q}(t) 
\end{bmatrix}
+\begin{bmatrix}
{\bf b}_1(t) \\
{\bf b}_2(t) 
\end{bmatrix}
\right)\,,
$$
where ${\bf p}, {\bf q} \in \mathbb{R}^{\frac{n}{2}}$,
obtaining the layer equations
\begin{equation}
\begin{cases}
    {\bf q}_{j+1} = {\bf q}_j - h \sigma({\bf K}_{0,j}^\top{\bf p}_j + {\bf b}_{j,1})\,,\\
    {\bf p}_{j+1} = {\bf p}_j + h \sigma({\bf K}_{0,j}{\bf q}_{j+1} + {\bf b}_{j,2})\,.
\label{eq:MS1}
\end{cases}
\end{equation}
Note that \eqref{eq:MS1} is an instance of H$_2$-DNN when
assuming ${\bf K}_j$ to be invertible 
and 
setting
${\bf J}_j  {\bf K}^\top_j = I_n$, and 
${\bf K}_j = 
\begin{bsmallmatrix}
\bf 0 & {\bf K}_{0,j}\\
-{\bf K}_{0,j}^\top & \bf 0
\end{bsmallmatrix}
$ for all $j=0,\dots,N-1$.

% MS2 DNN
\noindent\textbf{MS$_2$-DNN:} In \cite{Haber_2017,Chang19}, the authors propose to use FE to discretize
$$\dot{\bf y}(t) = \sigma({\bf K} (t) {\bf y}(t) + {\bf b} (t))\,,$$
where ${\bf K}(t)$ is skew-symmetric for all $ t\in [0,T]$, obtaining
the following layer equation 
\begin{equation}
{\bf y}_{j+1} = {\bf y}_j + h \sigma({\bf K}_j{\bf y}_{j} + {\bf b}_{j})\,.
\label{eq:MS2}
\end{equation}
In this case, \eqref{eq:MS2} is an instance of H$_1$-DNN by 
assuming ${\bf K}_j$ is invertible 
and by 
setting
${\bf J}_j {\bf K}^\top_j = I_n$ and 
${\bf K}_j = -{\bf K}^\top_j$ for all $j=0,\dots,N-1$.

% MS3 DNN
\noindent\textbf{MS$_3$-DNN:} In \cite{Chang18a}, the authors propose to use Verlet integration method to discretize
\begin{align*}
\begin{bmatrix}
\dot{\bf p}(t) \\
\dot{\bf q}(t)
\end{bmatrix}&=
\begin{bmatrix}
{\bf K}_1^\top (t)& 0_\frac{n}{2} \\
0_\frac{n}{2} & -{\bf K}_2^\top(t)
\end{bmatrix} \times\\
&
\times \sigma \left(
\begin{bmatrix}
0_\frac{n}{2} & {\bf K}_1(t)\\
{\bf K}_2 (t)& 0_\frac{n}{2}
\end{bmatrix} 
\begin{bmatrix}
{\bf p}(t) \\
{\bf q} (t)
\end{bmatrix}
+\begin{bmatrix}
{\bf b}_1 (t)\\
{\bf b}_2 (t)
\end{bmatrix}
\right)\,,
\end{align*}
where ${\bf p}, {\bf q} \in \mathbb{R}^{\frac{n}{2}}$,
obtaining the layer equations
\begin{equation}
\begin{cases}
{\bf p}_{j+1} = {\bf p}_j + h {\bf K}_{1,j}^\top \sigma({\bf K}_{1,j} {\bf q}_j + {\bf b}_{j,1})\,,\\
{\bf q}_{j+1} = {\bf q}_j - h {\bf K}_{2,j}^\top \sigma({\bf K}_{2,j} {\bf p}_{j+1} + {\bf b}_{j,2})\,.
\label{eq:MS3}
\end{cases}
\end{equation}
Note that \eqref{eq:MS3} is an instance of H$_2$-DNN when
setting
$${\bf K}_j = 
\begin{bmatrix}
0_\frac{n}{2} & {\bf K}_{1,j}\\
{\bf K}_{2,j} & 0_\frac{n}{2}
\end{bmatrix}
\,\text{ and }\,
{\bf J}_j  =
\begin{bmatrix}
0_{\frac{n}{2}} & I_{\frac{n}{2}}\\
-I_{\frac{n}{2}} & 0_{\frac{n}{2}}
\end{bmatrix}.$$

In \cite{Haber_2017} and \cite{Chang18a}, the MS$_1$- and MS$_3$-DNNs have been called \textit{Hamiltonian-inspired} in view of their similarities with Hamiltonian models,
although a precise Hamiltonian function for the corresponding ODE has not been provided. Moreover, note that the Verlet discretization used coincides with S-IE.

We highlight that a necessary condition for the skew-symmetric $n \times n$ matrix ${\bf K}_j$ to be invertible is that the size $n$ of input features is even.\footnote{For a $n\times n$ skew-symmetric matrix ${\bf A}$ we have, $\det({\bf A}) = \det({\bf A}^\top) = \det({\bf A}^{-1}) = (-1)^n \det({\bf A})$. If $n$ is odd, then $\det({\bf A}) = - \det({\bf A}) = 0$. Thus, ${\bf A}$ is not invertible.}
If $n$ is odd, however, one can perform input-feature augmentation by adding an extra state initialized at zero to satisfy the previous condition \cite{dupont2019augmented}.

\section{Proofs}
\subsection{Proof of Lemma~\ref{lem:grad_dynamics}}
\label{sec:lem:grad_dynamics}
Given the ODE~\eqref{eq:firstorderODE} with ${\bf y}(0) = {\bf y}_0$
we want to calculate the dynamics of $\frac{\partial {\bf y}(T)}{\partial {\bf y}(T-t)}$.

The solution to \eqref{eq:firstorderODE} can be expressed as
\begin{equation}
{\bf y}(t)={\bf y}(0) + \int_0^t {\bf f}({\bf y}(\tau), \boldsymbol{\theta}(\tau)) \,d \tau\,.
\label{eq:ap_solution}
\end{equation}
Analogously to Lemma~2 in \cite{ClaraL4DC}, evaluating \eqref{eq:ap_solution} in $t=T$ and $t=T-t-\delta$ and taking the limit of their ratio as $\delta\rightarrow 0$, we obtain
\begin{align} 
\frac{d}{d t} \frac{\partial {\bf y}(T)}{\partial {\bf y}(T-t)} = \frac{\partial {\bf f}}{\partial {\bf y}}\Big\rvert_{{\bf y}(T-t), \boldsymbol{\theta}(T-t)} \frac{\partial {\bf y}(T)}{\partial {\bf y}(T-t)}\,.
\end{align}

Since in our case ${\bf f}({\bf y}(t), \boldsymbol{\theta}(t)) = {\bf J}(t) {\bf K}(t)^\top \sigma( {\bf K}(t) {\bf y}(t) + {\bf b}(t) ) $, then, dropping the time-dependence for brevity, we have
\begin{align*}
\left. \frac{\partial {\bf f}}{\partial {\bf y}} \right\rvert_{{\bf y},\boldsymbol{\theta}} &=
\frac{\partial}{\partial {\bf y}} \left(  {\bf K}^\top \sigma({\bf K} {\bf y} +{\bf b})\right) {\bf J}^\top\\
&= \frac{\partial}{\partial {\bf y}} \left( \sigma({\bf K} {\bf y} +{\bf b})\right) {\bf K}{\bf J}^\top \nonumber \\
&= {\bf K}^\top \text{diag}\left(\sigma'({\bf K} {\bf y} +{\bf b})\right) {\bf K}{\bf J}^\top \\
&= {\bf K}^\top{\bf D}({\bf y}) {\bf K}{\bf J}^\top.
\end{align*}

 \subsection{Proof of Lemma~\ref{le:periodic}}
 \label{sec:le:periodic}
 
 To shorten the notation, let $\mathbf{s}(t,0,\mathbf{y}_0+\gamma \bm{\beta})=s_{\gamma,\beta}(t)$. We define the polar coordinates
\begin{align*}
    &\rho_{\gamma,\beta}(t) = \norm{s_{\gamma,\beta}(t)}_2\,, \quad \phi_{\gamma,\beta}(t) = \arctan\left(\frac{s_{\gamma,\beta,2}(t)}{s_{\gamma,\beta,1}(t)}\right)\,,
\end{align*}
 where $s_{\gamma,\beta,i}$ denotes the $i$-th entry of $s_{\gamma,\beta}$. To further streamline the notation, throughout this proof we denote $s_2(t)= s_{\gamma,\beta,2}(t)$, $s_1(t)= s_{\gamma,\beta,1}(t)$ and $\mathbf{s}(t) = s_{\gamma,\beta}(t)$.

$i)$  Assuming $\epsilon>0$ without loss of generality,
\begin{align}
    \dot{\phi}_{\gamma,\beta}(t) &= \frac{1}{1+\left(\frac{s_2}{s_1}\right)^2}\frac{\dot{s}_2 s_1-\dot{s}_1 s_2}{s_1^2}\nonumber\\
    &= \frac{\epsilon}{\rho_{\gamma,\beta}^2(t)} \cdot \mathbf{s}^\mathsf{T}\tanh(\mathbf{s})>0\,, \label{eq:phi_dot}
\end{align}
where the last equality holds because $\tanh(\cdot)$ preserves the sign. Since $\phi_{\gamma,\beta}(t)$ represents an angle in the Cartesian plane, the vector $\mathbf{s}(t)$ always rotates counter-clockwise around the origin. We conclude that there exists an instant $P_{\gamma,\beta} \in \mathbb{R}$ such that for all $t \in \mathbb{R}$
\begin{equation*}
    \phi_{\gamma,\beta}(P_{\gamma,\beta}+t) =\phi_{\gamma,\beta}(t)+2 \pi\,. 
\end{equation*}
Next, we prove $\rho_{\gamma,\beta}(t)=\rho_{\gamma,\beta}(P_{\gamma,\beta}+t)$ for all $t \in \mathbb{R}$. Let $H(t) = H (\mathbf{s}(t))$. Since $H(t)$ is constant for every $t \in \mathbb{R}$ because the ODE \eqref{eq:2D_Example} is time-invariant, we have 
\begin{align*}
    &H(0) = \log\cosh^\top(\mathbf{s}(t))1_2\\
    &=\log\cosh^\top\left(\rho_{\gamma,\beta}(t)\begin{bmatrix}\cos\phi_{\gamma,\beta}(t)\\\sin \phi_{\gamma,\beta}(t)\end{bmatrix}\right)1_2\\
    &= H(P_{\gamma,\beta}+t)\\
    &= \log\cosh^\top\left(\rho_{\gamma,\beta}(P_{\gamma,\beta}+t)\begin{bmatrix}\cos\phi_{\gamma,\beta}(P_{\gamma,\beta}+t)\\\sin \phi_{\gamma,\beta}(P_{\gamma,\beta}+t)\end{bmatrix}\right)1_2\\
    &= \log\cosh^\top\left(\rho_{\gamma,\beta}(P_{\gamma,\beta}+t)\begin{bmatrix}\cos\phi_{\gamma,\beta}(t)\\\sin \phi_{\gamma,\beta}(t)\end{bmatrix}\right)1_2\,.
\end{align*}

Let $\mathbf{v}= \begin{bmatrix}v_1\\v_2\end{bmatrix} = \begin{bmatrix}\cos\phi_{\gamma,\beta}(t)\\\sin \phi_{\gamma,\beta}(t)\end{bmatrix}$. We show that
\begin{align*}
    &\log\cosh(\rho_{\gamma,\beta}(t)\mathbf{v})^\mathsf{T}1_2 = \log\cosh(\rho(P_{\gamma,\beta}+t)\mathbf{v})^\mathsf{T}1_2 \\
    &\implies \rho_{\gamma,\beta}(t) = \rho_{\gamma,\beta}(P_{\gamma,\beta}+t)\,.
\end{align*}

Assume by contrapositive that $\rho(t)\neq \rho(P_{\gamma,\beta}+t)$, and without loss of generality that $\rho(P_{\gamma,\beta}+t)>\rho(t)>0$. The function $\log\cosh(x)$ is strictly monotonically decreasing for $x<0$, strictly monotonically increasing for $x>0$, and $\log\cosh(0) = 0$. Therefore
\begin{align*}\
   & \log\cosh(\rho_{\gamma,\beta}(t)v_1)+ \log\cosh(\rho_{\gamma,\beta}(t)v_2)<\\
   &<\log\cosh(\rho(P_{\gamma,\beta}+t)v_1)+ \log\cosh(\rho(P_{\gamma,\beta}+t)v_2)\,. 
\end{align*}
The above inequality holds true because $\log\cosh(\rho_{\gamma,\beta}(t)v_1)<\log\cosh(\rho_{\gamma,\beta}(P_{\gamma,\beta}+t)v_1)$ and $\log\cosh(\rho_{\gamma,\beta}(t)v_2)<\log\cosh(\rho_{\gamma,\beta}(P_{\gamma,\beta}+t)v_2)$, for any value of $x_1 \in \mathbb{R}$ and $x_2 \in \mathbb{R}$. Hence, we have reached a contradiction, and we deduce $\rho_{\gamma,\beta}(t)=\rho_{\gamma,\beta}(P_{\gamma,\beta}+t)$. We conclude that $\mathbf{y}(t) = \mathbf{y}(P_{\gamma,\beta}+t)$. Since the ODE \eqref{eq:2D_Example} is time-invariant, then $\mathbf{y}(kP_{\gamma,\beta}+t) = \mathbf{y}(t)$ for every $k \in \mathbb{N}$.

\vspace{0.2cm}

$ii)$ For any orientation of the Cartesian axes, let $t^\star$ and $t_\gamma^\star$ be time instants when $\phi_{0,\beta}(t^\star) =\phi_{\gamma,\beta}(t_\gamma^\star) = 0$. By \eqref{eq:phi_dot}
\begin{align*}
    \dot{\phi}_{\gamma,\beta}(t_\gamma^\star) = \frac{\epsilon s^\top(t_\gamma^\star) \tanh(t_\gamma^\star))}{\rho_{\gamma,\beta}^2(t_\gamma^\star)} =\frac{\epsilon \tanh(\rho_{\gamma,\beta}(t^\star)+h_{\gamma,\beta})}{\rho_{0,\beta}(t^\star)+h_{\gamma,\beta}}\,,
\end{align*}
where $h:\mathbb{R}_0^+\rightarrow \mathbb{R}_0^+$ is defined as $h_\beta(\gamma) = \rho_{\gamma,\beta}(t_\gamma^\star)-\rho_{0,\beta}(t^\star)$. We assume that $h_\beta(\cdot)$ increases with $\gamma>0$, that is, we assume that $\bm{\beta}$ points towards a higher sublevel set of the Hamiltonian; analogous reasoning holds if $\bm{\beta}$ points towards a lower sublevel set.  We have that $h_\beta(\gamma)$ is a continuous-function with $h_\beta(0) = 0$ and $h_\beta'(0)>0$.  Then
\begin{align*}
   &\frac{\partial}{\partial \gamma} \dot{\phi}_{\gamma,\beta}(t_\gamma^\star) =\\
    &~~\frac{\epsilon\Big(\sech^2(\rho_{0,\beta}(t^\star)+h_\beta(\gamma))h_\beta'(\gamma)(\rho_{0,\beta}(t^\star)+h_\beta(\gamma))-}{(\rho_{0,\beta}(t^\star)+h_\beta(\gamma))^2}\\
    &~~\frac{-h_\beta'(\gamma)\tanh(\rho_{0,\beta}(t^\star)+h_\beta(\gamma))\Big)}{(\rho_{0,\beta}(t^\star)+h_\beta(\gamma))^2}\,.
\end{align*}
The above evaluated at $\gamma = 0$ yields
\begin{align*}
    &\frac{\partial}{\partial \gamma} \dot{\phi}_{\gamma,\beta}(t^\star)\left.\right \rvert_{\gamma = 0} =\\
    &~~\epsilon h_\beta'(0)\frac{\sech^2(\rho_{0,\beta}(t^\star))\rho_{0,\beta}(t^\star)-\tanh(\rho_{0,\beta}(t^\star))}{\rho_{0,\beta}^2(t^\star)}\,.
\end{align*}
It can be verified that the value $\sech^2(x)x-\tanh(x)$ is negative for every $x>0$. We conclude that the angular time derivative decreases as the perturbation $\gamma$ increases. Since the above holds for any orientation of the Cartesian axes, we conclude that the period increases as $\gamma$ increases. Analogous reasoning holds if $\bm{\beta}$ points towards a lower sublevel set; in this case, the period decreases as $\gamma$ increases. 

\subsection{Proof of Proposition~\ref{lem:upper_bound}}
\label{sec:lem:upperbound}
We will first state the following Lemma needed for proving Proposition~\ref{lem:upper_bound}.

\begin{lemma}
	\label{lem:ap_norm_columns_to_matrix}
	Consider the matrix $\mathbf{A}\in \mathbb{R}^{n \times n}$ with columns $\mathbf{a}_j\in \mathbb{R}^n$ for $j=1,\dots,n$ (i.e., $\mathbf{A} = \begin{bmatrix}
	\mathbf{a}_1 & \mathbf{a}_2 & \dots & \mathbf{a}_n
	\end{bmatrix}$) and $\norm{\mathbf{a}_j}_2 \leq \gamma^+$ for all $j = 1,\dots,n$. Then, $\norm{\mathbf{A}}_2 \leq \gamma^+\sqrt{n}$.
\end{lemma}
\begin{proof}
	Consider $\mathbf{x} \in \mathbb{R}^n$ where the $i$-th entry of $\mathbf{x}$ is denoted as $x_i$ for $i = 1,\dots,n$. Since $\mathbf{A}\mathbf{x} = \sum_{i=1}^{n} x_i \mathbf{a}_i $, then,
	\begin{align}
          \norm{\mathbf{A}\mathbf{x}}_2& = \norm{\sum_{i=1}^{n} x_i \mathbf{a}_i}_2
                     \leq \sum_{i=1}^{n}  \norm{x_i \mathbf{a}_i}_2 \nonumber \\
                     &= \sum_{i=1}^{n}  \lvert x_i \rvert  \norm{\mathbf{a}_i}_2 
                      \leq \gamma^+ \sum_{i=1}^{n} \lvert x_i \rvert 
                       = \gamma^+ \norm{\mathbf{x}}_1\,.
	\label{eq:prop3_eq1}
	\end{align}
	
	It is easy to prove using the Cauchy-Schwarz inequality that $\frac{1}{\sqrt{n}} \norm{\mathbf{x}}_1 \leq \norm{\mathbf{x}}_2$.
	\label{eq:prop3_eq2}
	Then, we conclude:
	{
	\begin{multline*}
	\norm{\mathbf{A}}_2 = \sup_{\mathbf{x} \in \mathbb{R}^n} \frac{\norm{\mathbf{Ax}}_2}{\norm{\mathbf{x}}_2} \leq 
	\sup_{\mathbf{x} \in \mathbb{R}^n} \frac{ \gamma^+ \norm{\mathbf{x}}_1 }{ \frac{1}{\sqrt{n}} \norm{\mathbf{x}}_1 } \\
	= \sup_{\mathbf{x} \in \mathbb{R}^n}  \gamma^+ \sqrt{n}
	= \gamma^+ \sqrt{n}\,.
	\end{multline*}
	}
\end{proof}
We are now ready to prove Proposition~\ref{lem:upper_bound}.

\begin{proof}
Consider the backward gradient dynamics ODE \eqref{eq:BackwardGradientDynamics} and subdivide $\frac{\partial {\bf y}(T)}{\partial {\bf y}(T-t)}$ into columns as per  $\frac{\partial {\bf y}(T)}{\partial {\bf y}(T-t)} = \begin{bmatrix}{\bf z}_1(t) & {\bf z}_2(t) & \dots & {\bf z}_n(t)
\end{bmatrix}$.
Then, \eqref{eq:BackwardGradientDynamics} is equivalent to
\begin{equation}
\dot{\bf z}_i(t) = \mathbf{A}(T-t) {\bf z}_i(t)\,, \quad t \in [0,T]\,,
\label{eq:ap_BackwardGradientDynamics_z}
\end{equation}
for $ i=1,2\dots,n$ and $ {\bf z}_i(0) = e_i$, where $e_i$ is the unit vector with a single nonzero entry with value 1 (i.e. $e_1 = \begin{bmatrix}
	1 & 0 & \dots & 0
\end{bmatrix}^\top$).
The solution of \eqref{eq:ap_BackwardGradientDynamics_z} is given by
\begin{equation}
{\bf z}_i(t) = {\bf z}_i(0) + \int_{0}^{t} \mathbf{A}(T-\tau) {\bf z}_i(\tau) d \tau\,.
\label{eq:ap_z_solution}
\end{equation}
Assuming that 
$\|\mathbf{A}(T-t)\|_2 \leq Q$
for all $t \in [0,T]$, and applying the triangular inequality in \eqref{eq:ap_z_solution},  we obtain
\begin{align*}
\|{\bf z}_i(t)\|_2 &\leq \|{\bf z}_i(0)\|_2 + Q \int_{0}^{t} \|{\bf z}_i(\tau)\|_2 d \tau \\
&= 1 + Q \int_{0}^{t} \|{\bf z}_i(\tau)\|_2 d \tau\,,
\end{align*}
where the last equality comes from the fact that $\|{\bf z}_i(0)\|_2 = \|e_i\|_2=1$ for all $i=1,2,\dots,n$.
Then, applying Gronwall inequality, we have
\begin{equation}
\|{\bf z}_i(t)\|_2 \leq \exp(Q T)\,.
\label{eq:ap_z_exp}
\end{equation}
To characterize $Q$, note that
\begin{align*}
  &\|\mathbf{A}(T,T-t)\|_2 \\
  &=\| \mathbf{K}^\top(T-t) \mathbf{D}({\bf y}(T-t)) \mathbf{K}(T-t) \mathbf{J}^\top(T-t) \|_2\\
 &  \leq \hspace{-0.1cm} \| \mathbf{K}^\top(T-t) \|_2 \| \mathbf{D}({\bf y}(T-t)) \|_2 \| \mathbf{K}(T-t) \|_2 \| \mathbf{J}^\top(T-t) \|_2 \\
  &\leq \| \mathbf{K}^\top(T-t) \|_2^2\,  \| \mathbf{J}^\top(T-t) \|_2 \, S\sqrt{n}\,.
\end{align*}
The last inequality is obtained by applying Lemma~\ref{lem:ap_norm_columns_to_matrix} to $\mathbf{D}$ and noticing that each column of $\mathbf{D}$ is expressed as ${\bf d}_i = {e}_i \sigma'({\bf K}{\bf y}+{\bf b})_i$, where  $|\sigma'(x)|\leq S$ for every $x \in \mathbb{R}$ by assumption. Hence, we can characterize $Q$ as
\begin{equation*}
    Q = S\sqrt{n} \max_{t \in [0,T]} \| \mathbf{K}^\top(T-t) \|_2^2\,  \| \mathbf{J}^\top(T-t) \|_2\,.
\end{equation*}
Last, having bounded the column vectors ${\bf z}_i$, \eqref{eq:ap_z_exp}, we apply  Lemma~\ref{lem:ap_norm_columns_to_matrix} to conclude the proof.
\end{proof}

\subsection{Proof of Theorem~\ref{th:multi_agent}}
\label{sec:th:multi_agent}
Define 
\begin{align*}
\mathbf{v}^{[i]}(t) &= \sigma\left(\sum_{k|\mathbf{R}(i,k)= 1}\mathbf{K}^{i,k}(t)\mathbf{y}^{[k]}(t)+\mathbf{b}^{[i]}(t)\right)\,,\\
\mathbf{w}^{[i]}(t) &= \hspace{-0.5cm} \sum_{k|\mathbf{R}^\top(i,k)= 1}\hspace{-0.3cm}(\mathbf{K}^\top)^{i,k}(t) \mathbf{v}^{[k]}(t)\,, \\
\mathbf{z}^{[i]}(t) &= \hspace{-0.5cm} \sum_{k|\mathbf{T}(i,k) = 1} \hspace{-0.3cm}\mathbf{J}^{i,k}(t)\mathbf{w}^{[k]}(t)\,. 
\end{align*}
It is easy to verify that $\dot{\mathbf{y}}^{[i]}(t) = \mathbf{z}^{[i]}(t)$, where $\dot{\mathbf{y}}^{[i]}(t)$ indicates the forward H-DNN propagation update \eqref{eq:ODE_H} for node $i$.  
Next, observe that $\mathbf{v}^{[i]}(t)$ may depend on $\mathbf{y}^{[k]}(t)$ if and only if $\mathbf{R}(i,k) = 1$, $\mathbf{w}^{[i]}(t)$ may depend on $\mathbf{v}^{[l]}(t)$ if and only if $\mathbf{R}^\mathsf{T}(i,l) = 1$, and $\mathbf{z}^{[i]}(t)$ may depend on $\mathbf{w}^{[m]}(t)$ if and only $\mathbf{T}(i,m) = 1$. 
By dropping the time dependence to ease the notation, we deduce that $\dot{\mathbf{y}}^{[i]}(t)$ may depend on $\mathbf{y}^{[h]}(t)$ if and only if there exists two indices 
$r,s \in \{1,\dots, M\}$ such that $\mathbf{T}(i,r)=\mathbf{R}^\top(r,s) = \mathbf{R}(s,h) = 1$, or equivalently $(\mathbf{TR}^\top\mathbf{R})(i,h) = 1 $. We conclude that the forward propagation can be implemented in a localized way according to the graph $\mathbf{S}$ if $\mathbf{T}(t)\mathbf{R}(t)^\mathsf{T}\mathbf{R}(t) \leq \mathbf{S}$ holds. 

Similar reasoning holds for the backward propagation. Notice that the sparsity pattern of  $\mathbf{J}^\top(t)$ is the same as that of $\mathbf{J}(t)$ by skew-symmetricity. Define
\begin{equation*}
\begin{cases}
\mathbf{v}^{[i]}(T-t)= \sum_{k|\mathbf{T}(i,k)=1}(\mathbf{J}^\top)^{i,k}(T-t)\bm{\delta}^{[k]}(T-t)\,,\\
\mathbf{w}^{[i]}(T-t) = \sum_{k|\mathbf{R}(i,k)=1}\mathbf{K}^{i,k}(T-t)\mathbf{v}^{[k]}(T-t)\,,\\
   \mathbf{z}^{[i]}(T\hspace{-0.08cm}-\hspace{-0.08cm}t) =\\
   =\text{diag}\hspace{-0.1cm}\left(\hspace{-0.1cm}\sigma'\hspace{-0.1cm}\left(\sum\limits_{k|\mathbf{R}(i,k)=1}\hspace{-0.4cm}\mathbf{K}^{i,k}(T\hspace{-0.08cm}-\hspace{-0.08cm}t)\mathbf{y}^{[k]}(T\hspace{-0.08cm}-\hspace{-0.08cm}t)\text{+}\mathbf{b}^{[i]}(T\hspace{-0.08cm}-\hspace{-0.08cm}t)\hspace{-0.1cm}\right)\hspace{-0.2cm}\right)\hspace{-0.15cm} \times \\
   \qquad \qquad \quad \times \mathbf{w}^{[i]}(T-t)\,,\\
   \mathbf{u}^{[i]}(T-t) = \sum_{k|\mathbf{R}^\top(i,k) = 1}(\mathbf{K}^\top)^{i,k}(T-t)\mathbf{z}^{[k]}(T-t) \,.
   \end{cases}
\end{equation*}
Clearly, $\dot{\bm{\delta}}^{[i]}(T-t) = \mathbf{u}^{[i]}(T-t)$. Hence, $\dot{\bm{\delta}}^{[i]}(T-t)$ depends on $\mathbf{y}^{[k]}(T-t)$ if $(\mathbf{R}^\top(T-t)\mathbf{R}(T-t))(i,k) = 1$ and on $\bm{\delta}^{[l]}(T-t)$ if $(\mathbf{R}^\top(T-t) \mathbf{R}(T-t) \mathbf{T}(T-t))(i,l) = 1$. Since we have assumed that $\mathbf{T}(T-t) \geq I_M$, the sparsity of $(\mathbf{R}^\top(T-t) \mathbf{R}(T-t)$ is included in that of $\mathbf{R}^\top(T-t) \mathbf{R}(T-t) \mathbf{T}(T-t)$.

\subsection{Proof of Lemma \ref{lem:numerical_flow}}\label{ap:lem_numerical_flow}
	
We study the Hamiltonian system \eqref{eq:TV_HS} in the \textit{extended} phase space \cite{deGosson2011book}, i.e., 
we define an extended state vector $\tilde{\bf y} = ({\bf p}, {\bf q}, \varepsilon, t)$,\footnote{
Note that permuting the elements of $\tilde{\bf y}$, the state vector can be re-written as $(\tilde{\bf p}, \tilde{\bf q})$ where $\tilde{\bf p} = ({\bf p}, \varepsilon)$ and $\tilde{\bf q} = ({\bf q}, t)$.}
an extended interconnection matrix 
$\tilde{\bf J} = \begin{bsmallmatrix}
{\bf J} & 0_{n\times 2} \\
0_{2\times n}  & {\bf \Omega}
\end{bsmallmatrix}$, 
${\bf \Omega} = \begin{bsmallmatrix}
0 & -1 \\ 1 & 0
\end{bsmallmatrix}$ and
an extended Hamiltonian function $\tilde{H} = H({\bf p}, {\bf q}, t) + \varepsilon$, such that $\frac{d\varepsilon}{dt} = -\frac{dH}{dt}$. 

Note that the extended Hamiltonian system defined by $\tilde{H}$ is time-invariant by construction, i.e. $\frac{d \tilde{H}}{d {t}} = 0$. 
Then, following Theorem 3.3 in Section VI of \cite{Hairer2006book}, it can be seen that $\frac{\partial \tilde{\bf y}_{j+1}}{\partial \tilde{\bf y}_{j}}$ is a symplectic matrix with respect to $\tilde{\mathbf{J}}$, i.e. it satisfies
\begin{equation}
\label{eq:symplectic_extended}
\left[\frac{\partial \tilde{\bf y}_{j+1}}{\partial \tilde{\bf y}_{j}}\right]^\top \tilde{\bf J} \left[\frac{\partial \tilde{\bf y}_{j+1}}{\partial \tilde{\bf y}_{j}}\right] = \tilde{\bf J}\,.
\end{equation}
Next, we show that \eqref{eq:symplectic_extended} implies symplecticity for the BSM of the original time-varying system \eqref{eq:implicit_euler}. 
First, we introduce the S-IE layer equations for the extended Hamiltonian dynamics:
\begin{equation*}
\begin{cases}
{\bf p}_{j+1} = {\bf p}_{j} - h {\bf X}^\top \frac{\partial H}{\partial {\bf q}}({\bf p}_{j+1}, {\bf q}_j, t_j)\,, \\ 
{\bf q}_{j+1} = {\bf q}_{j} + h {\bf X} \frac{\partial H}{\partial {\bf p}}({\bf p}_{j+1}, {\bf q}_j, t_j)\,, \\ 
\varepsilon_{j+1} = \varepsilon_j - h \frac{\partial H}{\partial t}({\bf p}_{j+1}, {\bf q}_j, t_j)\,, \\ 
t_{j+1} = t_j + h\,.
\end{cases}
\end{equation*}
Then, differentiation with respect to $\tilde{\bf y}_j = \left( {\bf p}_j, {\bf q}_j, \varepsilon_j, t_j \right)$ yields\footnote{To improve readability, the dimension of the zero matrices has been omitted.}
{
\begin{multline}
\frac{\partial \tilde{\bf y}_{j+1}}{\partial \tilde{\bf y}_j}
\left(
I_n - h 
\begin{bmatrix}
H_{pp} & H_{qp} & 0 & H_{tp} \\
0 & 0 & 0 & 0 \\
0 & 0 & 0 & 0 \\
0 & 0 & 0 & 0 \\
\end{bmatrix}  
\begin{bmatrix}
{\bf J} & 0 \\
0  & {\bf \Omega}
\end{bmatrix}^\top 
\right)
= \\
\left(
I_n + h 
\begin{bmatrix}
0 & 0 & 0 & 0 \\
H_{pq} & H_{qq} & 0 &  H_{tq}\\
0 & 0 & 0 & 0 \\
H_{pt} & H_{qt} & 0 & H_{tt} \\
\end{bmatrix}
\begin{bmatrix}
{\bf J} & 0 \\
0  & {\bf \Omega}
\end{bmatrix}^\top
\right)\,,
\label{eq:differentation_Htv}
\end{multline}
}
where $$H_{xy} = \frac{\partial H ({\bf p}_{j+1}, {\bf q}_{j},t_{j})}{\partial x \partial y}\,,$$ and $x,y$ indicate any combination of two variables in the set $\{{\bf p},{\bf q},t\}$. 
It remains to verify that \eqref{eq:differentation_Htv} implies symplecticity of
\begin{equation*}
    \frac{\partial {\bf y}_{j+1}}{\partial {\bf y}_{j}} = \begin{bmatrix} \frac{\partial{\bf p}_{j+1}}{\partial{\bf p}_{j}} & \frac{\partial{\bf q}_{j+1}}{\partial{\bf p}_{j}} \\ \frac{\partial{\bf p}_{j+1}}{\partial{\bf q}_{j}} & \frac{\partial{\bf q}_{j+1}}{\partial{\bf q}_{j}} \end{bmatrix}\,.
\end{equation*}
By denoting $\boldsymbol{\Gamma} =  \left(I_n - h \begin{bsmallmatrix} H_{pp} & H_{qp} \\ 0 & 0 \end{bsmallmatrix}{\bf J}^\top\right)$
	and $\boldsymbol{\Lambda} =\left(I_n + h \begin{bsmallmatrix} 0 & 0 \\ H_{pq} & H_{qq} \end{bsmallmatrix}	{\bf J}^\top \right)$, where $\bm{\Gamma}$ and $\bm{\Lambda}$ are invertible for almost every choice of step size $h$, the part of \eqref{eq:differentation_Htv} concerning $\frac{\partial {\bf y}_{j+1}}{\partial {\bf y}_{j}}$ reads as
 	$$\frac{\partial {\bf y}_{j+1}}{\partial {\bf y}_{j}} \boldsymbol{\Gamma} =  \boldsymbol{\Lambda}\,.$$
The above implies
\begin{equation*}
    \frac{\partial {\bf y}_{j+1}}{\partial {\bf y}_{j}}^\top \mathbf{J} \frac{\partial {\bf y}_{j+1}}{\partial {\bf y}_{j}} = \mathbf{J} \iff  \boldsymbol{\Lambda}^\mathsf{T} {\bf J} \boldsymbol{\Lambda}  = \boldsymbol{\Gamma}^\mathsf{T} {\bf J} \boldsymbol{\Gamma}\,,
\end{equation*}
where the equality $ \boldsymbol{\Lambda} {\bf J} \boldsymbol{\Lambda}  = \boldsymbol{\Gamma} {\bf J} \boldsymbol{\Gamma}$ is  verified by direct inspection. We conclude that $\frac{\partial {\bf y}_{j+1}}{\partial {\bf y}_{j}}$ is a symplectic matrix with respect to $\mathbf{J}$.

\section{Numerical validation of the analysis of \eqref{eq:2D_Example}}
\label{sec:num_validation}
We simulate the system \eqref{eq:2D_Example} for $\epsilon = 1$. In Figure~\ref{fig:phi_TTt_trajectories} we compute the value 
\begin{equation*}\frac{\norm{\mathbf{s}(T,T-\tau,\mathbf{y}_0+\gamma \bm{\beta})-\mathbf{s}(T,T-\tau,\mathbf{y}_0)}}{\gamma}\,,
\end{equation*}
for $\gamma = 0.005$, $\bm{\beta} = \begin{bmatrix}1\\0\end{bmatrix},\begin{bmatrix}0\\1\end{bmatrix}$. This is obtained by selecting $\mathbf{y}(0)$ at random, letting $\mathbf{y}_0 = \mathbf{s}(T-t,0,\mathbf{y}(0))$ and then computing both $\mathbf{s}(T,T-t,\mathbf{y}_0 + \gamma \bm{\beta})$ and $\mathbf{s}(T,T-t,\mathbf{y}_0)$ by appropriate numerical integration of \eqref{eq:2D_Example}. 

\begin{figure}[h]
	\centering
	\includegraphics[width=0.8\linewidth]{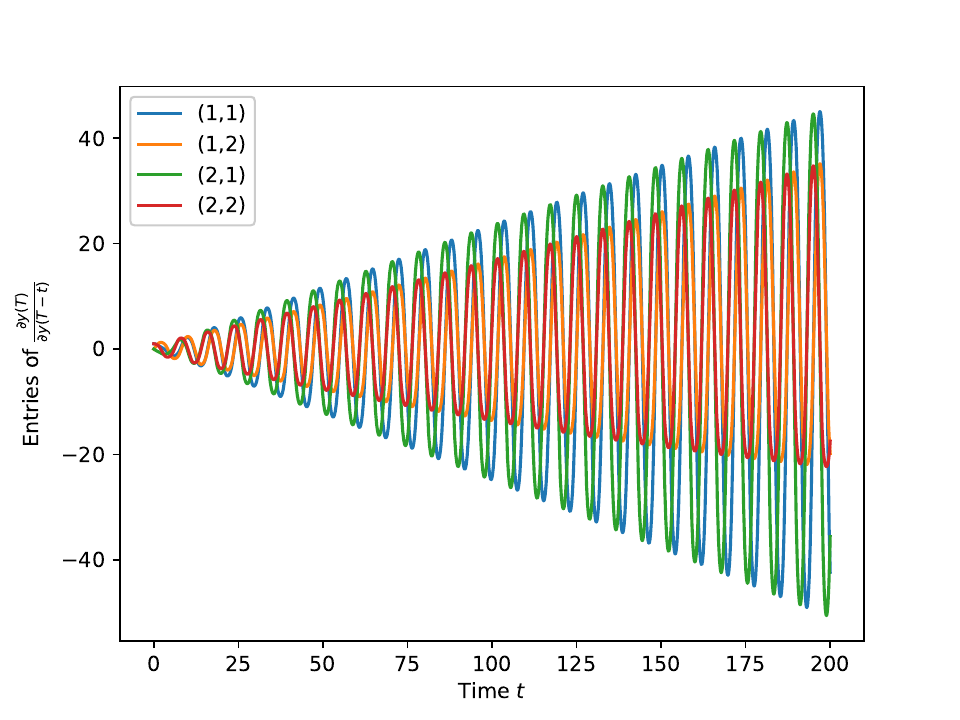}
	\caption{Temporal evolution of $\frac{\partial {\bf y}(T)}{\partial{\bf y}(T-t)}$ for \eqref{eq:2D_Example}.}
	\label{fig:phi_TTt_trajectories}
\end{figure}

The numerical experiment confirms that the entries of $\frac{\partial {\bf y}(T)}{\partial{\bf y}(T-t)}$ diverge as $t$ increases. We note that for a large enough $t$ the values reach a maximum value, as expected because $\gamma$ is not infinitesimal. One can further inspect that, when $\gamma$ is chosen closer to $0$, the maximum values achieved by the trajectory diverges. 

\section{Implementation details}

DNN architectures and training algorithms are implemented using the PyTorch library.\footnote{\url{https://pytorch.org/}}

\subsection{Binary classification datasets}\label{ap:implementation_2d}
For two-class classification problems, we use 8000 datapoints and a mini-batch size of 125, for both training and test data. 
Training is performed using coordinate gradient descent, i.e. a modified version of stochastic gradient descent (SGD) with Adam ($\beta_1 = 0.9$, $\beta_2 = 0.999$) \cite{Haber_2017}. 
Following \cite{Haber_2017}, in every iteration of the algorithm, first the optimal weights of the output layer are computed given the last updated parameters of the hidden layers, and then, a step update of the hidden layers' parameters is performed by keeping fixed the output parameters. 
The training consists of 50 epochs and each of them has a maximum of 10 iterations to compute the output layer weights. 
The learning rate, or optimization step size as per $\gamma$ in \eqref{eq:GD_update}, is set to $2.5 \times 10^{-2}$. 
For the regularization, $\alpha_\ell =0 $, the weight decay for the output layer ($\alpha_N$) is constant and set to $1 \times 10^{-4}$ and $\alpha$ is set to $5 \times 10^{-4}$.

\begin{figure*}[htbp]
    \centering
    \includegraphics[width=0.99\linewidth]{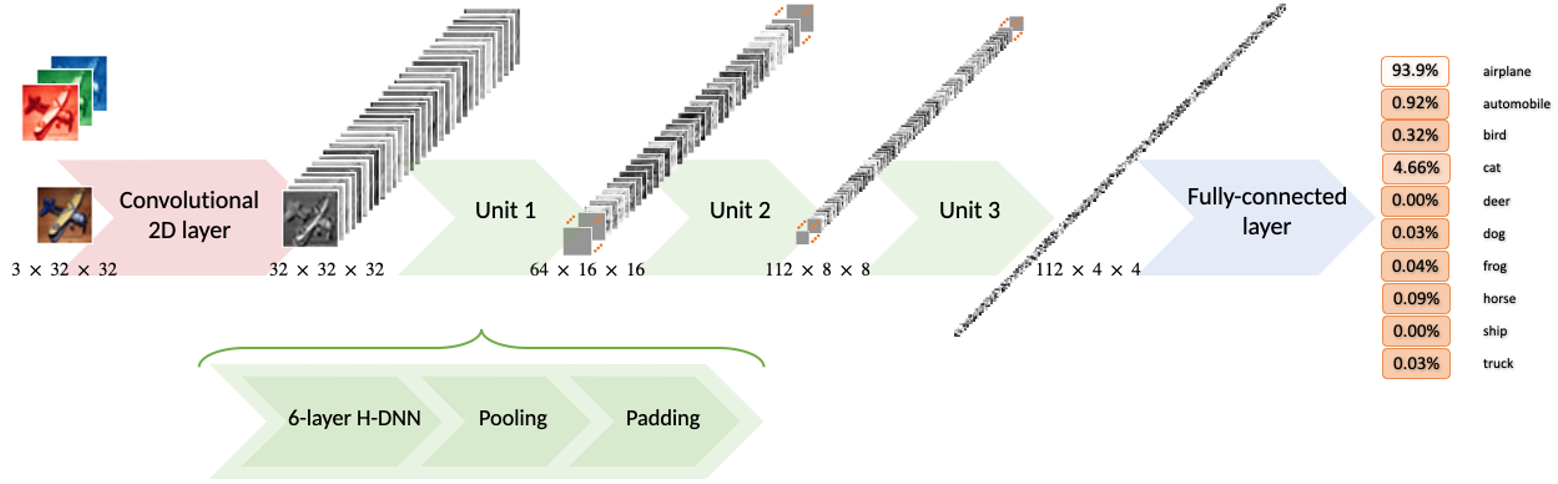}
    \caption{Scheme of the H$_1$-DNN$^*$ architecture used for classification over the CIFAR-10 dataset.}
    \label{fig:hdnn_cifar10}
\end{figure*}

\subsection{MNIST dataset}\label{ap:implementation_mnist}
We use the complete MNIST dataset (60,000 training examples and 10,000 test examples) and a mini-batch size of 100. 
For the optimization algorithm we use SGD with Adam \cite{kingma2015adam} and cross-entropy loss.
The learning rate, or optimization step size as per $\gamma$ in \eqref{eq:GD_update}, is initialized to be 0.04 with a decay rate of 
$0.8$ at each epoch. The total training step is 40 epochs. 
For MS$_1$-DNN, we set $\alpha_N = \alpha_\ell = 1\times 10^{-3}$ and $\alpha = 1\times 10^{-3}$. For H$_1$-DNN, we set $\alpha_N = \alpha_\ell = 4\times 10^{-3}$ and $\alpha = 8\times 10^{-3}$.

\section{Experiments on CIFAR-10 with enhanced H-DNNs}\label{ap:CIFAR}

We test an enhanced version of the H-DNN architecture, namely H-DNN$^*$, over a more complex dataset: CIFAR-10.\footnote{\url{https://www.cs.toronto.edu/~kriz/cifar.html}} It consists of 60000  RGB color images of $32\times32$ pixels equally distributed in 10 classes. There are 50000 training images and 10000 test images.
We summarize the accuracy results in Table~\ref{tab:cifar10} where we show that the performance of H$_1$-DNN$^*$ is comparable with state-of-the-art DNNs: AlexNet and ResNet architectures. Notice that while obtaining similar test accuracies, the AlexNet model requires using much more trainable parameters. In the next subsection, we present the details of the specific architecture of H$_1$-DNN$^*$.
    \begin{table}[htbp]
	\begin{center}
		\caption{Classification accuracies over the test set for the CIFAR-10 example when using three different architectures. The number of parameters of each architecture is provided.}
		\label{tab:cifar10}
		\begin{tabular}{|c|c|c|c|}
			\hline
			Model & \# of layers & \# of parameters (M) & Test accuracy (\%)\\ \hline
			\hline
			AlexNet & 8 & 57 & 91.55 \\  \hline
			H$_1$-DNN$^*$ & 20 & 0.97 & 92.27 \\ \hline
            ResNet-56 & 56 & 0.85 & 93.68  \\
			\hline
		\end{tabular}
	\end{center}
    \end{table}

\subsection{Implementation details of H$_1$-DNN$^*$}

The architecture of H$_1$-DNN$^*$ is similar to the one used in~\cite{Chang18a}. The parameter $\mathbf{K}$ of the H-DNN is now a convolutional operator with a filter of dimension $3\times 3$.

We summarize, in~Figure \ref{fig:hdnn_cifar10}, the architecture utilized for the H$_1$-DNN$^*$. 
First, we define a convolutional layer that increases the channels from $3$ to $32$. 
Then, three Hamiltonian units are concatenated in sequential order. 
Finally, a fully connected layer is added in order to obtain the probability of the image belonging to each of the ten classes.

Each Hamiltonian unit consists of a 6-layer H$_1$-DNN with ReLU activation function followed by an average-pooling layer and a zero-padding layer. The former performs a downsampling of the image and the latter increases the number of channels by adding zeros.
The first unit receives a 32-channel $32\times 32$ pixel image, the second, a 64-channel $16 \times 16$ image, and the last one, a 112-channel $8\times 8$ image.
The output of the last unit consists of a matrix of dimension $112 \times 4 \times 4$.
Note that this last unit does not contain the padding layer.
Prior to feeding the $112 \times 4 \times 4$ matrix into the fully connected layer, a flatten operation is performed.

The training is done using stochastic gradient descent with momentum ($\beta = 0.9$) over 160 epochs using a batch size of 100 images. We set the initial learning rate to $0.1$ and reduce it by 10 times at epochs 120, 140, and 150. We set $h = 0.1$ and $\alpha_N = \alpha_\ell = \alpha = 2\times 10^{-4}$.

\end{document}